\documentclass[letterpaper]{article} 
\usepackage[preprint]{aaai2027}
\usepackage[hyphens]{url}  
\usepackage{graphicx} 
\graphicspath{{figures/}}
\usepackage{natbib}  
\usepackage{caption} 
\usepackage{algorithm}
\usepackage{algpseudocode} 
\usepackage{amsthm}

\usepackage{newfloat}
\usepackage{adjustbox}
\usepackage{listings}
\DeclareCaptionStyle{ruled}{labelfont=normalfont,labelsep=colon,strut=off} 
\floatstyle{ruled}
\newfloat{listing}{tb}{lst}{}
\floatname{listing}{Listing}

\usepackage{booktabs}
\usepackage{amssymb}

\newtheorem{lemma}{Lemma}
\newtheorem{proposition}{Proposition}
\newtheorem{definition}{Definition}
\newtheorem{remark}{Remark}

\title{GraspMeanFlow: SE(3)-Equivariant MeanFlow for Few-Step 6-DoF Grasp Generation}
\author{
    Jiyong Kwon\textsuperscript{\rm 1}\equalcontrib,
    Yikun Bai\textsuperscript{\rm 2}\equalcontrib\corresponding,
    Amirhossein Mollaali\textsuperscript{\rm 1},
    Guang Lin\textsuperscript{\rm 1,\rm 2}\corresponding
}

\affiliations{
    \textsuperscript{\rm 1}School of Mechanical Engineering,
    Purdue University\\
    \textsuperscript{\rm 2}Department of Mathematics,
    Purdue University\\
    West Lafayette, IN 47907, USA\\
    kwon165@purdue.edu,
    bai195@purdue.edu,
    amollaal@purdue.edu,
    guanglin@purdue.edu
}

\usepackage{amsmath}
\usepackage{amsfonts}
\usepackage{etoc}

\begin{document}

\pagestyle{plain}
\pagenumbering{arabic}
\setcounter{page}{1}

\maketitle
\thispagestyle{plain}
\begin{abstract}
Recent data-driven methods for synthesizing 6-DoF grasp poses use generative
models to learn complex grasp pose distributions and generate diverse candidate
poses. In particular, SE(3)-equivariant flow-based models generate grasp poses
that transform consistently with object rotations and translations. However,
these methods sample by iterative numerical integration, requiring tens of
function evaluations per grasp and limiting their use in real-time manipulation.
We propose GraspMeanFlow, an SE(3)-equivariant MeanFlow framework for few-step
6-DoF grasp generation. Our method learns the average velocity over a finite
time interval, defined through the time-ordered exponential so that it
reproduces exactly the rigid-body displacement accumulated over that interval.
We prove that a point-cloud-conditioned distribution transported by an
equivariant average-velocity flow map remains invariant, so equivariance is
retained under few-step sampling, and we condition the field on a pair of times
by lifting both to equivariant vectors, leaving the backbone otherwise
unchanged. For stable training, we pair a flow-matching boundary term with either of two
consistency terms: the differential MeanFlow identity, whose target requires a
Jacobian-vector product, or an equivalent semigroup loss that avoids it. Experiments on ACRONYM show that a single function evaluation of GraspMeanFlow
reaches the EMD that an iterative SE(3) flow model needs five steps to approach, 
that a second instantiation of the same framework improves grasp success by up
to 24.3 points in the few-step regime, and that both generate grasp
distributions transforming exactly with the object.
\end{abstract}
\section{Introduction}

Synthesizing six degrees of freedom (DoF) grasp poses from 3D observations of an
object, such as a surface point cloud $\mathcal{P}$, is a fundamental task in
robotics. A grasp pose for a parallel-jaw gripper is a rigid transformation
$T=(R,x)\in\mathrm{SE}(3)$, and since the set of successful grasps is highly
multimodal, the task is naturally posed as conditional generative modeling of
$p(T\mid\mathcal{P})$ on the $\mathrm{SE}(3)$ manifold. Data-driven generative
approaches have accordingly become dominant, from variational
autoencoders~\citep{mousavian20196} and energy-based
diffusion~\citep{urain2023se3dif} to stochastic-interpolant
bridges~\citep{chen2024bridger}. Most recently, EquiGraspFlow
\citep{lim2024equigraspflow} proposed a conditional continuous normalizing flow
whose architecture guarantees $\mathrm{SE}(3)$-equivariance by construction, so
that generated grasps transform identically with the object without any data
augmentation.

However, a limitation shared by all of these models is that sampling is
\emph{iterative}. Flow-based models numerically integrate a time-dependent
velocity field and diffusion-based models repeatedly denoise, so a single sample
costs tens of network evaluations; EquiGraspFlow, for instance, uses a
fourth-order Runge--Kutta--Munthe-Kaas solver over $20$ steps, amounting to $80$
evaluations of the velocity field. Since a manipulation pipeline generates,
filters, and ranks a hundred candidates every time the scene changes, an ideal
model should reproduce the same grasp distribution in one or a few evaluations.

MeanFlow~\citep{geng2025meanflow,geng2025improved} offers a route to such a model, replacing the
instantaneous velocity of flow matching with the \emph{average velocity} over a
finite interval and yielding a training objective for few-step generation that
requires neither a teacher nor distillation. Grasp poses, however, do not live
in a Euclidean vector space, and the extension is not merely notational. Because
$\mathfrak{so}(3)$ is non-commutative, the average angular velocity cannot be
obtained by integrating instantaneous velocities, as the naive integral does not
reproduce the endpoint rotation of the trajectory. Flattening rotations into a
Euclidean representation sidesteps this difficulty, but forfeits both the
geometry of the manifold and the equivariance that makes $\mathrm{SE}(3)$ grasp
generators robust to object rotations.

In this paper, we propose \emph{GraspMeanFlow}, an $\mathrm{SE}(3)$-equivariant
few-step 6-DoF grasp pose generative model. We define the average velocity on
$\mathrm{SE}(3)$ through the \emph{time-ordered} exponential, so that the
learned Lie-algebra element reproduces exactly the rigid-body displacement
accumulated over an interval, and develop the conditional and equivariant theory
that grasp generation requires but that existing manifold MeanFlow formulations
do not address. Realizing this model raises two further difficulties---a MeanFlow
field is conditioned on a \emph{pair} of times, which equivariant point-cloud
backbones cannot accommodate through the usual scalar embeddings, and the
differential MeanFlow identity is numerically fragile on $\mathrm{SO}(3)$---both
of which we resolve while leaving the equivariance guarantee intact. Our
contributions are as follows:

    \begin{itemize}
    \item \textbf{Conditional and Equivariant SE(3) MeanFlow.} We formulate
    average-velocity generation for point-cloud-conditioned distributions on
    $\mathrm{SE}(3)$, defining the average velocity through the time-ordered exponential. We prove that a
    conditional distribution transported by an equivariant average-velocity flow
    map remains invariant, so equivariance survives the reduction in sampling
    steps.

    \item \textbf{Equivariance-Preserving Network for Two-Time Fields.} We
    extend the Vector Neuron~\citep{deng2021vector} backbone of EquiGraspFlow
    from a single time to a time pair $(s,t)$ while retaining its
    $\mathrm{SO}(3)$-equivariance, by lifting both times through its
    existing equivariant lifting mechanism. The extension
    adds no parameters beyond the embedding, so our comparison isolates the
    training objective from model capacity.

\item \textbf{Two Instantiations and Practical Recipes.} The interval field can
be trained by any consistency term anchored to a flow-matching boundary. We
report two: \textbf{GMF-SG}, which uses a finite semigroup identity and avoids
the Jacobian--vector product entirely, and \textbf{GMF-JVP}, which uses the
differential identity in an inverse-Jacobian target form together with a
per-object optimal-transport coupling. They are complementary --- GMF-SG is
stronger on distributional fidelity, GMF-JVP on simulated grasp success --- and
on four ACRONYM categories improve success by up to $24.3$ points in the
few-step regime.
\end{itemize}

\section{Background}
\subsection{Rigid Transformations and Grasp Poses}\label{sec:se3-basics}
The special Euclidean group is
\begin{equation}
    \mathrm{SE}(3)=\{(R,x): R\in\mathrm{SO}(3),\;x\in\mathbb{R}^3\},\nonumber
\end{equation}
acting on a point $\mathbf{p}\in\mathbb{R}^3$ by $\mathbf{p}\mapsto R\mathbf{p}+x$, and we represent a grasp pose as $T=(R,x)$. For an angular velocity $\omega\in\mathbb{R}^3$ and a linear velocity $v\in\mathbb{R}^3$, a rigid-pose trajectory satisfies
\begin{equation}
    \dot R_t = [\omega_t]\, R_t, \qquad \dot x_t = v_t,
    \label{eq:rigid-ode}
\end{equation}
where $[\cdot] : \mathbb{R}^3 \to \mathfrak{so}(3)$ maps a three-vector to the skew-symmetric matrix satisfying $[\omega]\,\mathbf{p} = \omega \times \mathbf{p}$ for any $\mathbf{p} \in \mathbb{R}^3$:
\begin{equation}
    [\omega] = \begin{pmatrix} 0 & -\omega_3 & \omega_2 \\ \omega_3 & 0 & -\omega_1 \\ -\omega_2 & \omega_1 & 0 \end{pmatrix}, \qquad \omega = (\omega_1, \omega_2, \omega_3)^\top .\nonumber
\end{equation}
Its inverse $(\cdot)^\vee : \mathfrak{so}(3) \to \mathbb{R}^3$ recovers the three-vector from a skew-symmetric matrix. We identify $\mathfrak{so}(3)\cong\mathbb{R}^3$ through this pair and write angular velocities as three-vectors throughout, the matrix form $[\omega]$ appearing only inside $\exp$ and $\mathcal{T}\!\exp$.

\subsection{Flow matching for 6-DoF grasp generation}\label{sec:grasp-fm}

A flow-based grasp generator constructs a probability path from a simple prior distribution $p_1(T\mid\mathcal{P})$ to the data distribution $p_0(T\mid\mathcal{P})$. For training pairs $(T_0, T_1)$ with $T_0$ a successful grasp and $T_1$ a prior sample, we define an interpolation
\begin{align}
    R_t &= \exp\!\left(t\,\log(R_1R_0^\top)\right) R_0, \label{eq:geodesic_R}\\
    x_t &= (1-t)x_0 + t x_1.\label{eq:geodesic_x}
\end{align}
The corresponding instantaneous velocity under this convention is
\begin{equation}
    \omega_{0\rightarrow 1} = \log(R_1R_0^\top)^\vee,\qquad v_{0\rightarrow 1}=x_1-x_0.\nonumber 
\end{equation}
A standard flow-matching model learns a time-dependent vector field $(\omega_\theta,v_\theta)(T_t,t,\mathcal{P})$ and samples by numerical integration from $t=1$ to $t=0$.

\subsection{MeanFlow and Finite-Interval Average Velocity}
\label{sec:meanflow-bg}

MeanFlow~\citep{geng2025meanflow} predicts the finite-interval \emph{average
velocity} in place of the instantaneous velocity of flow matching. For a
Euclidean trajectory $z_\tau\in\mathbb{R}^n$ with $\dot z_\tau=v_\tau$, the
average velocity over $[s,t]$ is
\begin{equation}
    u_{s,t}\;=\;\frac{1}{t-s}\int_s^t v_\tau\,d\tau,
    \label{eq:euclid-avg}
\end{equation}
which the model predicts from the state $z_t$ and the time pair $(s,t)$.
Differentiating $(t-s)\,u_{s,t}$ in $t$ gives the \emph{MeanFlow identity}
\begin{equation}
    v_t\;=\;u_{s,t}\;+\;(t-s)\,\frac{d}{dt}\,u_{s,t},
    \label{eq:x-identity}
\end{equation}
whose second term is a Jacobian--vector product; regressing the left-hand side
against the model output is the MeanFlow training objective.

Two concurrent works~\citep{zhong2026riemannian,woo2026riemannian} extend MeanFlow to Riemannian manifolds, evaluated on synthetic geometries. We instead develop MeanFlow on $\mathrm{SE}(3)$ through its closed form Lie group structure (Section~\ref{sec:se3-meanflow}), integrating spatial-frame angular and linear velocities in the Lie algebra; for the geodesic interpolation used here, the average velocity has a closed form through the Lie logarithm, which serves as the training target.
The same $\mathrm{SE}(3)$ average-velocity formulation is applied to protein backbone generation by \citet{bai2026se3meanflow}, in the body frame $\dot R_t=R_t[\omega_t]$ rather than the spatial frame used here, which our equivariance guarantee requires (Section~\ref{sec:equivariance}).

\section{Our Method: SE(3)-MeanFlow for 6-DoF Grasp Generation}\label{sec:se3-meanflow}

\begin{figure}
    \centering
    \includegraphics[width=1\linewidth]{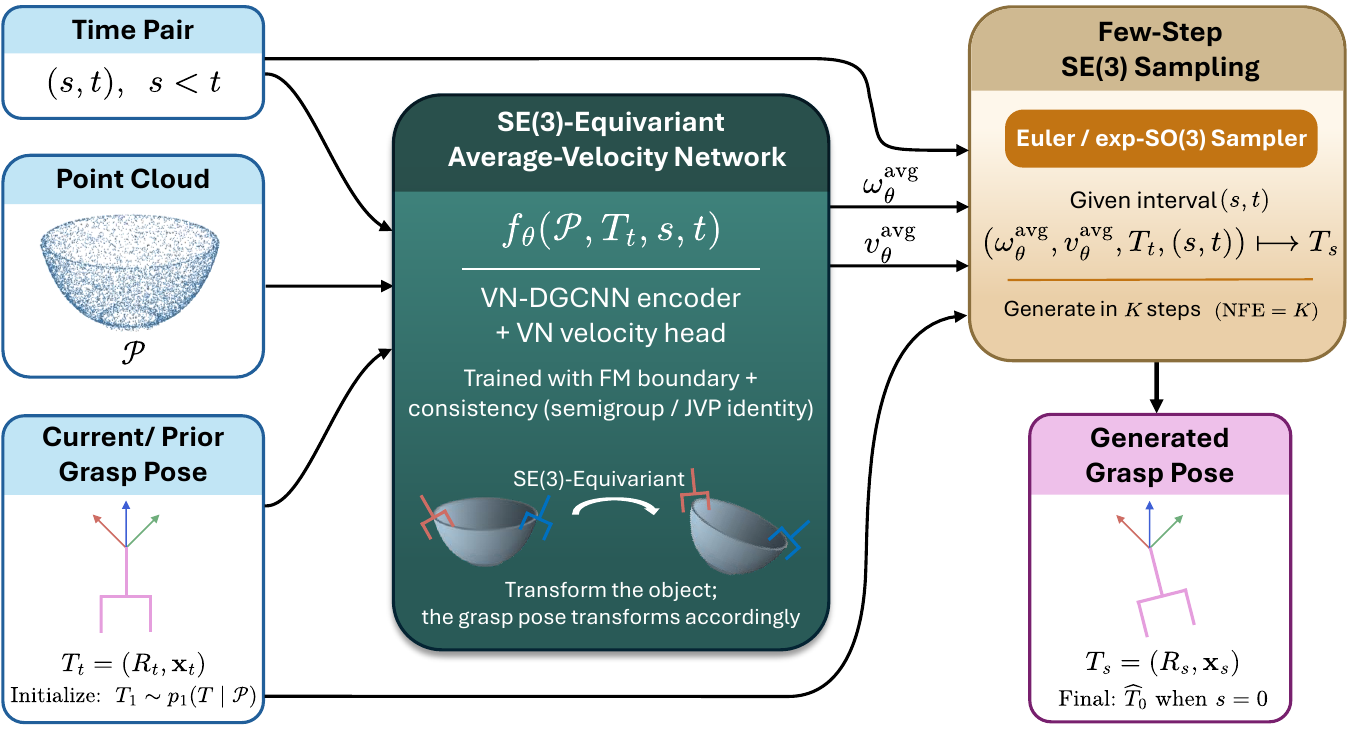}
    \caption{\textbf{Overview of the proposed GraspMeanFlow framework.}
    Given an object point cloud $\mathcal{P}$, a current grasp pose $T_t$ initialized from the object-centered prior at $t=1$, and a time pair $(s,t)$ with $s<t$, the SE(3)-equivariant average-velocity network predicts the interval-averaged angular and linear velocities $(\omega_{\theta}^{\mathrm{avg}},\,{v}_{\theta}^{\mathrm{avg}})$. An Euler or exp-$SO(3)$ sampler then uses these predictions to transport $T_t$ to the earlier-time pose $T_s$ and generate a grasp pose in a few steps. The network is trained using a flow-matching boundary loss together with a consistency term (semigroup or differential/JVP), while preserving SE(3)-equivariance.}
    \label{fig:main_fig}
\end{figure}

\subsection{Problem Formulation}\label{sec:problem}
\paragraph{Point cloud, pose and dataset.}
Given an object point cloud $\mathcal{P}:=\{\mathbf{p}_k\in\mathbb{R}^3\}_{k=1}^K$, we aim to sample a set of candidate grasps
\begin{equation}
    \{T^{(m)}=(R^{(m)},x^{(m)})\}_{m=1}^M \sim p(T\mid\mathcal{P}),\nonumber
\end{equation}
where $R\in \mathrm{SO}(3), x\in \mathbb{R}^3$. At training time, the dataset
\begin{equation}
    \mathcal{D}=\{(\mathcal{P}_i,\{T_{ij}\}_{j=1}^{J_i})\}_{i=1}^N\nonumber
\end{equation}
contains $N$ objects, where each point cloud $\mathcal{P}_i$ is annotated with $J_i$ successful grasp poses $\{T_{ij}\}$ drawn from the ground-truth conditional distribution $q(T\mid\mathcal{P}_i)$. Our generator is trained on these positive grasps to approximate $q$. Figure~\ref{fig:main_fig} gives an overview of the framework.

\paragraph{Conditional prior and interpolation.}
We use an object-centered prior $p_1(T\mid\mathcal{P})=p_1(R)\,p_1(x\mid\mathcal{P})$, with $p_1(R)$ uniform on $\mathrm{SO}(3)$ and $p_1(x\mid\mathcal{P})$ Gaussian centered at the point-cloud mean $\bar{\mathbf{p}}=\frac{1}{K}\sum_k \mathbf{p}_k$:
\begin{equation}
    R_1\sim \mathcal{U}(\mathrm{SO}(3)), \qquad x_1\sim\mathcal N(\bar{\mathbf{p}}, \sigma^2 I).\nonumber
\end{equation}
We use $\sigma=1$ throughout, so that in the mean-centered frame in which the
network operates the translation prior is standard normal. Centering at $\bar{\mathbf{p}}$ makes the prior $\mathrm{SE}(3)$-invariant, which is a prerequisite for the equivariance guarantee of Section~\ref{sec:equivariance}. Given a data grasp $T_0=(R_0,x_0)$ and a prior sample $T_1=(R_1,x_1)$, the intermediate pose $T_t$ is obtained by geodesic interpolation on $\mathrm{SO}(3)$ and linear interpolation in $\mathbb{R}^3$.

\subsection{Average Velocity on SE(3)}
\label{sec:avg-velocity}

We decompose the flow on $\mathrm{SE}(3)$ into its $\mathrm{SO}(3)$ and
$\mathbb{R}^3$ components. The rotation branch is where non-commutativity of the
group enters and is our focus below; the translation branch is the abelian
degenerate case and coincides with Euclidean MeanFlow.

Consider a rigid-pose trajectory $T_\tau=(R_\tau,x_\tau)$ on $\mathrm{SE}(3)$
governed by instantaneous angular and linear velocities
$(\omega_\tau,v_\tau)\in\mathbb{R}^3\times\mathbb{R}^3$,
\begin{equation}
    \dot R_\tau=[\omega_\tau]\,R_\tau,
    \qquad
    \dot x_\tau=v_\tau .
    \label{eq:se3-ode}
\end{equation}
For an interval $[s,t]$ with $s<t$, we define the $\mathrm{SE}(3)$ average
velocity $(\omega^{\mathrm{avg}},v^{\mathrm{avg}})(s,t,T_t)$ as the constant
Lie-algebra velocity whose endpoint update over $[s,t]$ matches that of the
trajectory:
\begin{align}
    \exp\!\big((t-s)\,[\omega^{\mathrm{avg}}]\big)
    &=\mathcal{T}\!\exp\!\Big(\int_s^t [\omega_\tau]\, d\tau\Big),
    \label{eq:se3-mean-velocity}\\
    (t-s)\,v^{\mathrm{avg}}
    &=\int_s^t v_\tau\, d\tau .
    \label{eq:se3-mean-velocity-x}
\end{align}
Here $\mathcal{T}\!\exp$ denotes the time-ordered matrix exponential, required
because $\mathfrak{so}(3)$ is non-commutative: integrating $\omega_\tau$ as a
vector does not in general recover the endpoint rotation. Integrating
\eqref{eq:se3-ode} gives $R_t=D(s,t)\,R_s$ with
$D(s,t)=R_tR_s^\top=\exp\!\big((t-s)[\omega^{\mathrm{avg}}]\big)$, so that
$(\omega^{\mathrm{avg}},v^{\mathrm{avg}})$ encode the rigid-body displacement
accumulated over $[s,t]$ in a single Lie-algebra element. We suppress the
dependence on $\mathcal{P}$ where unambiguous.

\paragraph{Differential form.}
Along the geodesic interpolation of Section~\ref{sec:grasp-fm}, the instantaneous
angular velocity is the constant $\omega_t=\log(R_1R_0^\top)^\vee$, so that
$R_t=\exp(t\,[\omega_t])\,R_0$. Writing
\begin{equation}
    \omega^{s\to t}:=(t-s)\,\omega^{\mathrm{avg}}(s,t,T_t)\in\mathbb{R}^3
    \label{eq:log-displacement}
\end{equation}
for the rotation log-displacement, so that
$\exp\!\big([\omega^{s\to t}]\big)=D(s,t)$, and differentiating
\eqref{eq:se3-mean-velocity} with respect to $t$ yields the $\mathrm{SE}(3)$
counterpart of the MeanFlow identity,
\begin{equation}
    J\big(\omega^{s\to t}\big)\;\tfrac{d}{dt}\omega^{s\to t}\;=\;\omega_t ,
    \label{eq:omega_identity}
\end{equation}
where $\tfrac{d}{dt}$ is the total derivative along the trajectory,
\begin{equation}
    \tfrac{d}{dt}\omega^{s\to t}
    =\partial_t \omega^{s\to t}
    +\big\langle\nabla_R \omega^{s\to t},\dot R_t\big\rangle
    +\big\langle\nabla_x \omega^{s\to t},\dot x_t\big\rangle ,
    \label{eq:domega-dt}
\end{equation}
and $J:\mathbb{R}^3\to\mathbb{R}^{3\times3}$ is the left Jacobian of
$\mathrm{SO}(3)$,
\begin{equation}
    J(\phi):=\int_0^1 e^{\alpha[\phi]}\,d\alpha
    =\sum_{k=0}^{\infty}\frac{1}{(k+1)!}\big([\phi]\big)^{k} .
    \label{eq:left_jacobian}
\end{equation}
Its closed form and small-$\|\phi\|$ expansion are given in
Appendix~\ref{sec:se3-mf-details}\footnote{All appendices referenced in this paper are provided in the technical supplement.}, together with a derivation of
\eqref{eq:omega_identity}.

Equation~\eqref{eq:omega_identity} is the differential characterization of
the definition \eqref{eq:se3-mean-velocity} and suggests a regression target in the manner of Euclidean MeanFlow. 

\begin{remark}[Translation branch]
Since $\mathbb{R}^3$ is abelian, \eqref{eq:se3-mean-velocity-x} is a plain time
integral, the left Jacobian satisfies $J\equiv I$, and
\eqref{eq:omega_identity} collapses to the Euclidean MeanFlow identity
\eqref{eq:x-identity}. We therefore treat the translation branch exactly as
in MeanFlow~\citep{geng2025meanflow}.
\end{remark}

\subsection{Training Objective}
\label{sec:objective}

\paragraph{The differential objective.}
Equation~\eqref{eq:omega_identity} provides a regression target for the average
velocity, in direct analogy with Euclidean MeanFlow. Let $f_\theta$ output
$(\omega^{\mathrm{avg}}_\theta,v^{\mathrm{avg}}_\theta)(s,t,T_t)$ and write
$\omega^{s\to t}_\theta:=(t-s)\,\omega^{\mathrm{avg}}_\theta$. Solving
\eqref{eq:omega_identity} for the average velocity and applying a stop-gradient,
as in~\citep{geng2025meanflow}, gives the target
\begin{equation}
    \omega^{\mathrm{avg}}_{\mathrm{tgt}}
    =J\big(\omega^{s\to t}_\theta\big)^{-1}\omega_t
     -(t-s)\,\tfrac{d}{dt}\omega^{\mathrm{avg}}_\theta ,
    \label{eq:diff-target}
\end{equation}
and the $\mathrm{SE}(3)$ MeanFlow objective
\begin{equation}
    \mathcal{L}^{\mathrm{diff}}
    =\mathbb{E}_{s<t}\Big[\big\|\omega^{\mathrm{avg}}_\theta
      -\mathrm{sg}\big(\omega^{\mathrm{avg}}_{\mathrm{tgt}}\big)\big\|_2^2\Big],
    \label{eq:diff-loss}
\end{equation}
where $\tfrac{d}{dt}\omega^{\mathrm{avg}}_\theta$ is computed as a single
Jacobian--vector product along the tangent
$(\dot R_t,\dot x_t,\dot t,\dot s)=([\omega_t]R_t,\,v_t,\,1,\,0)$. The
translation branch is the $J\equiv I$ case, recovering the Euclidean MeanFlow
target $v_t-(t-s)\tfrac{d}{dt}v^{\mathrm{avg}}_\theta$ verbatim.

In practice, however, we find the Jacobian--vector product in
\eqref{eq:diff-loss} to be numerically unstable on $\mathrm{SO}(3)$. We therefore turn to an equivalent
characterization of the same definition that avoids it.

\paragraph{Semigroup property.}
Write $D(s,t)=R_tR_s^\top=\exp\!\big([\omega^{s\to t}]\big)$ for the
rotation accumulated over $[s,t]$ and
$p^{s\to t}:=(t-s)\,v^{\mathrm{avg}}(s,t,T_t)$ for the translation
displacement. Inserting $R_m^\top R_m=I$ gives, for any intermediate
$m\in[s,t]$,
\begin{equation}\label{eq:sg-identity}
\begin{cases}
\exp\!\big([\omega^{s\to t}]\big)
    =\exp\!\big([\omega^{m\to t}]\big)\,
     \exp\!\big([\omega^{s\to m}]\big),\\
p^{s\to t}=p^{s\to m}+p^{m\to t}.
\end{cases}
\end{equation}
The rotation composition is \emph{exact} and curvature-free: it is nothing but
the associativity of group multiplication. Curvature enters only if one collapses
the product into a single Lie-algebra sum, where Baker--Campbell--Hausdorff cross
terms appear; we never do, and therefore never need $J$.

\paragraph{Consistency loss.}
Given $m\sim\mathcal{U}([s,t])$, we form the intermediate state by stepping back
the near segment $[m,t]$,
\begin{equation}
    R_m=\exp\!\big(-[\omega^{m\to t}_\theta]\big)R_t,
    \qquad
    x_m=x_t-p^{m\to t}_\theta,
    \label{eq:sg-intermediate}
\end{equation}
and regress the model's \emph{direct} one-step prediction on $[s,t]$ onto the
stop-gradient \emph{composed} two-step target given by \eqref{eq:sg-identity}:
{
\begin{align}
\mathcal{L}^{\mathrm{sg}}
    &=\mathbb{E}\Big[\big\|\omega^{s\to t}_\theta
       -\mathrm{sg}\big(\omega^{s\to t}_{\mathrm{tgt}}\big)\big\|_2^2
      +\big\|p^{s\to t}_\theta
       -\mathrm{sg}\big(p^{s\to t}_{\mathrm{tgt}}\big)\big\|_2^2\Big],
    \label{eq:sg-loss}\\
    \omega^{s\to t}_{\mathrm{tgt}}
    &=\log\!\Big(\exp\big([\omega^{m\to t}_\theta]\big)\,
       \exp\big([\omega^{s\to m}_\theta]\big)\Big)^{\!\vee},
    \nonumber\\
    p^{s\to t}_{\mathrm{tgt}}&=p^{s\to m}_\theta+p^{m\to t}_\theta .
    \nonumber
\end{align}}

Equivalently, the same constraint can be imposed on the transported
\emph{states} rather than the displacements,
\begin{equation}
\mathcal{L}^{\mathrm{sg}}_{\mathrm{geo}}
=\mathbb{E}\Big[\big\|\log(R_dR_c^\top)^{\vee}\big\|_2^2
+\big\|x_d-x_c\big\|_2^2\Big],
\label{eq:sg-loss-geo}
\end{equation}
where $(R_d,x_d)$ and $(R_c,x_c)$ are the poses reached from $(R_t,x_t)$ by the
direct update on $[s,t]$ and by the composed update through $m$. The two
residuals agree to first order and share the same minimizer, differing only in
how they weight errors away from it (Appendix~\ref{sec:semigroup}); we train
with \eqref{eq:sg-loss-geo}. The state-space form \eqref{eq:sg-loss-geo} is
also arrived at, independently and concurrently, by
\citet{woo2026riemannian} as a manifold flow-map consistency objective; we
reach it instead as the finite, derivative-free counterpart of the differential
$\mathrm{SE}(3)$ MeanFlow identity \eqref{eq:omega_identity}, which is what lets
us drop both the Jacobian--vector product and the left Jacobian $J$.
 
 Neither form carries a $1/(t-s)$ factor, since both compare bounded quantities
rather than velocities; the normalization scalar must in any case stay outside
$\log(\exp\cdot\exp)$, since rescaling the two non-commuting
segment generators would corrupt the BCH cross term.

\paragraph{Boundary anchor.}
The semigroup term alone admits collapsed minimizers, so we anchor it at the
diagonal $s\to t$ with a flow-matching boundary,
\begin{equation}
    \mathcal{L}^{\mathrm{bd}}
    =\mathbb{E}_t\Big[\big\|\omega^{\mathrm{avg}}_\theta(t,t,T_t)-\omega_t\big\|_2^2
     +\big\|v^{\mathrm{avg}}_\theta(t,t,T_t)-v_t\big\|_2^2\Big],
    \label{eq:bd-loss}
\end{equation}
giving the total objective
$\mathcal{L}=\lambda_{\mathrm{bd}}\mathcal{L}^{\mathrm{bd}}
+\lambda_{\mathrm{sg}}\mathcal{L}^{\mathrm{sg}}$
(Appendix~\ref{sec:semigroup}).

\paragraph{$\alpha$-Flow warm-up.}
The consistency loss \eqref{eq:sg-loss} uses the model's own prediction on the
near segment $[m,t]$, which is unreliable early in training. We therefore warm up
with an $\mathrm{SE}(3)$ $\alpha$-Flow objective~\citep{cheng2025alpha}, a special
case of \eqref{eq:sg-loss} in which the near segment is supplied by the
\emph{data} velocity rather than the model. We describe the rotation branch; the
translation branch is the Euclidean $\alpha$-Flow objective
of~\citet{cheng2025alpha} verbatim. Place the split point at
$m=\alpha s+(1-\alpha)t$ with $\alpha\in[\alpha_{\min},1]$, so that the near
segment spans $[m,t]$ of length $\alpha(t-s)$. Its accumulated rotation is taken
from the data velocity, $D(m,t)=\exp\big(\alpha(t\!-\!s)[\omega_t]\big)$,
while the far segment $[s,m]$ uses the stop-gradient model evaluation at the
stepped-back state, $D(s,m)=\exp\big([\omega^{s\to m}_\theta]\big)$. By the
semigroup identity \eqref{eq:sg-identity}, the composed target is
\begin{equation}
    \omega^{s\to t}_{\mathrm{tgt}}
    =\log\!\Big(\exp\big(\alpha(t\!-\!s)[\omega_t]\big)\,
      \exp\big([\omega^{s\to m}_\theta]\big)\Big)^{\!\vee},
    \label{eq:af-target-main}
\end{equation}
which the model regresses against as in \eqref{eq:sg-loss}. Like the semigroup
loss, this uses only $\exp$ and $\log$ on $\mathrm{SO}(3)$ and no
Jacobian--vector product. The parameter $\alpha$ interpolates between flow
matching at $\alpha=1$, where the near segment fills the interval and
\eqref{eq:af-target-main} reduces to the boundary term \eqref{eq:bd-loss}, and the
differential target \eqref{eq:diff-loss} as $\alpha\to0$; we anneal $\alpha$
downward over warm-up before switching to the semigroup objective. Full
derivations, including the $\alpha\to0$ limit, are in
Appendix~\ref{sec:alpha-flow}.

\subsection{SE(3)-Equivariant Conditional Generation}
\label{sec:equivariance}

An element $T'=(R',x')\in\mathrm{SE}(3)$ acts on a point cloud by
$T'\mathcal{P}:=\{R'\mathbf{p}_k+x'\}_{k=1}^K$ and on a grasp pose by
$T'T:=(R'R,\,R'x+x')$. Equivariant grasp generation means that rigidly
transforming the object transforms the generated grasps identically, which in
the language of distributions is the following invariance.

\begin{definition}[Invariant conditional distribution]
\label{def:invariant}
A distribution on $\mathrm{SE}(3)$ conditioned on a point cloud is
$\mathrm{SE}(3)$-invariant if $p(T'T\mid T'\mathcal{P})=p(T\mid\mathcal{P})$ for
all $T'\in\mathrm{SE}(3)$.
\end{definition}

\begin{definition}[Equivariant average-velocity field]
\label{def:equi-field}
The pair $(\omega^{\mathrm{avg}}_\theta,v^{\mathrm{avg}}_\theta)$ is
$\mathrm{SE}(3)$-equivariant if, for all $T'=(R',x')$,
\begin{equation}
    (\omega^{\mathrm{avg}}_\theta,v^{\mathrm{avg}}_\theta)(s,t,T'\mathcal{P},T'T)
    =R'\,(\omega^{\mathrm{avg}}_\theta,v^{\mathrm{avg}}_\theta)(s,t,\mathcal{P},T).
\end{equation}
\end{definition}

Sampling applies the average-velocity step (see details in Appendix~\ref{sec:inference}), which
for $s\le t$ defines the map
\begin{equation}
    \Phi_{t\to s}(\mathcal{P},T_t)
    =\Big(\exp\!\big(-(t\!-\!s)[\omega^{\mathrm{avg}}_\theta]\big)R_t,\;
          x_t-(t\!-\!s)\,v^{\mathrm{avg}}_\theta\Big).
    \label{eq:flow-map}
\end{equation}

\begin{proposition}[Equivariance of the flow map]
\label{prop:flowmap-equi}
If $(\omega^{\mathrm{avg}}_\theta,v^{\mathrm{avg}}_\theta)$ is equivariant, then
$\Phi_{t\to s}(T'\mathcal{P},T'T_t)=T'\,\Phi_{t\to s}(\mathcal{P},T_t)$, and so is
any composition $\Phi_{t_1\to t_0}\circ\cdots\circ\Phi_{t_K\to t_{K-1}}$.
\end{proposition}
\begin{proposition}[Few-step invariance]
\label{prop:fewstep-inv}
Let the prior $p_1(T\mid\mathcal{P})$ be $\mathrm{SE}(3)$-invariant and let
$\Phi_{\mathcal{P}}$ be any composition of equivariant steps
\eqref{eq:flow-map}. Then the generated distribution
$(\Phi_{\mathcal{P}})_{\#}p_1$ is $\mathrm{SE}(3)$-invariant.
\end{proposition}
We refer to Appendix~\ref{sec:mf-theory} for detailed proofs.

\begin{remark}
Proposition~\ref{prop:fewstep-inv} is stated for pushforward measures rather than
densities, which is what makes it valid in the few-step regime: the
change-of-variables argument used for continuous normalizing
flows~\citep{katsman2021equivariant,lim2024equigraspflow} requires the transport
map to be a diffeomorphism, whereas a learned one-step map carries no such
guarantee. Invariance of the generated distribution needs only equivariance of
the map.
\end{remark}

Both hypotheses hold in our construction: the prior of Section~\ref{sec:problem} is invariant because the rotation component is uniform
on $\mathrm{SO}(3)$ and the translation component is centered at
$\bar{\mathbf{p}}$, which transforms with the object; and the two-time velocity
network (See Appendix~\ref{sec:arch}) is $\mathrm{SE}(3)$-equivariant, which we prove
in Proposition~\ref{prop:net-equivariant}. The training targets
inherit equivariance as well --- \eqref{eq:sg-intermediate} and
\eqref{eq:sg-loss} are built from equivariant field evaluations composed by
$\exp$, $\log$ and addition, all of which commute with the group action
--- so the objective is invariant, and training does not break the guarantee.

\subsection{Sampling}
\label{sec:sampling}

Because $f_\theta$ already represents the interval-averaged field, sampling
replaces numerical integration by a direct average-velocity step. Given a prior
sample $T_1=(R_1,x_1)\sim p_1(T\mid\mathcal{P})$, few-step generation applies
this update on each subinterval of a schedule $1=t_K>\cdots>t_0=0$,
\begin{align}
&R_{t_{k-1}}=\exp\!\big(-(t_k\!-\!t_{k-1})[\omega^{\mathrm{avg}}_\theta]\big)R_{t_k},\nonumber\\
&x_{t_{k-1}}=x_{t_k}-(t_k\!-\!t_{k-1})\,v^{\mathrm{avg}}_\theta ,\label{eq:few-step}
\end{align}
so that the number of function evaluations equals the number of steps; one-step
generation is the case $K=1$. We report
results for two samplers built on \eqref{eq:few-step}: Euler integration as
written, and an exp-$\mathrm{SO}(3)$ variant that takes the rotation update
toward the endpoint implied by the interval-averaged field; pseudocode for both is given in Appendix~\ref{sec:inference}.

Guidance is applied at no additional cost. Following
EquiGraspFlow~\citep{lim2024equigraspflow}, the point cloud is replaced by a
null condition with fixed probability during training, and the network models
the \emph{guided} average velocity directly rather than combining two
evaluations at sampling time. Being an affine combination of the conditional and
unconditional fields, the guided field is equivariant whenever both are, so
Propositions~\ref{prop:flowmap-equi} and~\ref{prop:fewstep-inv} apply unchanged.

\begin{table*}[h]
\centering
\setlength{\tabcolsep}{2pt}
\renewcommand{\arraystretch}{0.95}
\begin{adjustbox}{max width=\textwidth}
\begin{tabular}{l cccc c cccc c cccc c cccc}
\toprule
& \multicolumn{4}{c}{Bowl} && \multicolumn{4}{c}{Laptop} && \multicolumn{4}{c}{Mug} && \multicolumn{4}{c}{Pencil} \\
\cmidrule{2-5}\cmidrule{7-10}\cmidrule{12-15}\cmidrule{17-20}
NFE & 2 & 5 & 10 & 20 && 2 & 5 & 10 & 20 && 2 & 5 & 10 & 20 && 2 & 5 & 10 & 20 \\
\midrule
SE(3)-DiF & 14.1 & 21.7 & 34.3 & 52.3 && 8.0 & 14.9 & 22.9 & 34.5 && 10.4 & 19.9 & 31.0 & 44.6 && 11.4 & 23.6 & 36.2 & 51.2 \\
BRIDGER & $\mathbf{22.0}$ & 77.8 & 90.9 & 93.0 && $\mathbf{9.8}$ & 56.4 & 79.0 & 84.1 && $\mathbf{40.0}$ & $\mathbf{90.6}^{*}$ & 93.0 & 94.4 && 5.4 & 72.6 & $\mathbf{98.8}^{*}$ & $\mathbf{99.0}$ \\
EquiGraspFlow & $\mathbf{58.6}^{*}$ & 75.1 & $95.1$ & 98.0 && $\mathbf{25.6}^{*}$ & 61.2 & $\mathbf{85.5}$ & $\mathbf{95.8}^{*}$ && $\mathbf{57.8}^{*}$ & 85.1 & 91.7 & $\mathbf{97.5}^*$ && $\mathbf{65.0}^{*}$ & $\mathbf{81.3}$ & 92.5 & $\mathbf{99.6}^{*}$ \\
\midrule
GMF-SG (exp-$\mathrm{SO}(3)$) & 15.0 & $\mathbf{93.6}$ & $\mathbf{98.7}$ & $\mathbf{99.2}^{*}$ && 8.2 & $\mathbf{82.6}$ & 83.6 & 85.6 && 23.4 & 87.8 & $\mathbf{94.8}$ & $\mathbf{95.0}$ && $\mathbf{30.4}$ & 77.9 & 85.3 & 84.8 \\
GMF-JVP (exp-$\mathrm{SO}(3)$) & 10.7 & $\mathbf{96.8}^{*}$ & $\mathbf{98.9}^{*}$ & $\mathbf{99.0}$ && 6.3 & $\mathbf{85.5}^{*}$ & $\mathbf{90.5}^{*}$ & $\mathbf{87.6}$ && 17.5 & $\mathbf{88.1}$ & $\mathbf{97.5}^{*}$ & $\mathbf{97.5}^{*}$ && 26.9 & $\mathbf{94.1}^{*}$ & $\mathbf{98.2}$ & 98.1 \\
\midrule
Ground truth & \multicolumn{4}{c}{97.7} && \multicolumn{4}{c}{97.1} && \multicolumn{4}{c}{96.7} && \multicolumn{4}{c}{99.7} \\
\bottomrule
\end{tabular}
\end{adjustbox}
\caption{Per-category grasp success rate (\%) vs.\ NFE, with the
exp-$\mathrm{SO}(3)$ sampler for both GraspMeanFlow variants. Higher is
better. Best and second-best in each column are highlighted; the best is marked with $^{*}$.}
\label{tab:success}
\end{table*}

\section{Experiments}
\label{sec:experiments}
We evaluate GraspMeanFlow on conditional 6-DoF grasp generation against recent diffusion and flow-based baselines, reporting both generation quality and sampling efficiency as the step budget is reduced. We evaluate two instantiations of the framework, which share the network and training budget and differ in the consistency term and the prior--data coupling: \textbf{GMF-SG} pairs the flow-matching boundary with the semigroup loss under independent coupling, and \textbf{GMF-JVP} pairs it with the differential MeanFlow identity, supervised through a Jacobian--vector product, under per-object optimal-transport coupling.

\subsection{Datasets}
Following EquiGraspFlow~\citep{lim2024equigraspflow}, we conduct experiments on four object categories from ACRONYM~\citep{eppner2021acronym}: Bowl, Laptop, Mug, and Pencil, with their train/test splits. Each training example consists of an object point cloud and a set of successful 6-DoF parallel-jaw grasp poses. A single model is trained jointly on all four categories, and every method is evaluated on the same randomly rotated test objects.

\subsection{Baselines}
We compare with three released models. EquiGraspFlow learns an instantaneous $\mathrm{SE}(3)$-equivariant flow and is the most closely related baseline: our model shares its backbone and benchmark, so the comparison isolates the training objective. SE(3)-DiffusionFields~\citep{urain2023se3dif} generates grasps with a score-based diffusion process on $\mathrm{SE}(3)$. BRIDGER~\citep{chen2024bridger} transports samples from a heuristic prior with a stochastic-interpolant bridge. Earlier and excluded grasp generators are discussed in Appendix~\ref{sec:related}.

\subsection{Implementation Details}
We use the publicly released checkpoints of all baselines under a consistent evaluation protocol; SE(3)-DiffusionFields and BRIDGER are trained on supersets of our categories. The baselines are integrated with first-order Euler solvers at step counts matched to the reported NFE; EquiGraspFlow's released RK-MK sampler would otherwise spend four field evaluations per step, and the diffusion baselines are additionally reported at their native settings. Evaluation seeds are fixed and shared across methods.

\subsection{Training Details}
\paragraph{Model and trainer.}
We implement our method within the publicly released EquiGraspFlow codebase and
depart from it in three respects. \emph{(i) Model.} The field is conditioned on
a time pair $(s,t)$ through a shared two-time embedding; the VN-DGCNN encoder
and vector-neuron head are otherwise inherited unchanged
(Appendix~\ref{sec:arch}). \emph{(ii) Objective.} In place of the flow-matching
loss, we train with the SE(3) MeanFlow objective in decomposed form: a
flow-matching boundary term together with a consistency term. We report two instantiations, which share this identical network --- GMF-SG uses semigroup consistency, GMF-JVP the differential identity \eqref{eq:omega_identity} in its inverse-Jacobian form (Appendix~\ref{sec:algo}). \emph{(iii) Coupling.} GMF-JVP additionally pairs
each data grasp with a prior sample by a per-object $\mathrm{SE}(3)$
optimal-transport assignment rather than by independent sampling
(Appendix~\ref{sec:ot-coupling}); GMF-SG retains the independent coupling.

\paragraph{Two-stage training.}
Stage~1 is an $\alpha$-Flow warm-up (the first 18k steps, with $\alpha$ annealed $1{\to}0.2$); Stage~2 switches to the boundary-plus-consistency objective for the remaining 102k steps. The warm-up is optional (Appendix~\ref{sec:ablation}). The consistency weight of GMF-JVP is set so that the two loss terms contribute comparably. Full hyperparameters for both configurations and all stages are listed in Appendix~\ref{sec:config}.

\begin{figure}
    \centering
    \includegraphics[width=0.9\linewidth]{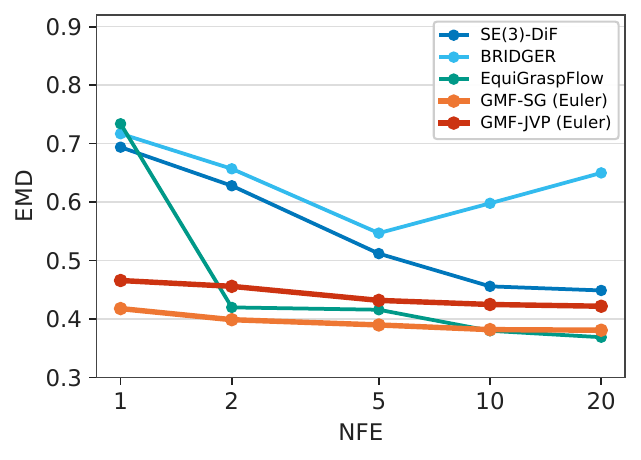}
    \caption{EMD versus NFE, averaged over the four categories, with Euler integration. Both GraspMeanFlow configurations are within $0.05$ of their converged EMD after a single evaluation, whereas EquiGraspFlow collapses at NFE~1 and the diffusion-based baselines stay above $0.44$ throughout.}
\vspace{-1em}
\label{fig:emd_nfe}
\end{figure}

\subsection{Evaluation Metrics and Settings}
Distribution fidelity is measured by Hungarian EMD under the benchmark $\mathrm{SE}(3)$ pose cost. Grasp feasibility is measured by simulated lift success under an NVIDIA Isaac Gym~\citep{makoviychuk2021isaac} protocol, identical for every method. Inference efficiency is measured by the number of function evaluations (NFE) and model-only wall-clock latency; for every sampler used here, each integration step uses one velocity-field evaluation, so NFE equals the number of sampling steps, and the two configurations share the same network and therefore the same cost per step. All latency measurements use a single A100 with batch size one and include model inference only. EMD is reported with Euler integration for both configurations. Success is reported with the exp-$\mathrm{SO}(3)$ sampler, an endpoint-style sampler on the interval-averaged field: at each step it forms the data-endpoint prediction from the average velocity and moves the rotation along the geodesic toward it at a fixed rate, following the exponential rotation scheduler of ReQFlow~\citep{yue2025reqflow}. The accelerated schedule applies only to the rotational component; the translational update is always linear, and the sampler requires NFE${\geq}2$. Success under Euler integration is reported in Appendix~\ref{sec:dense}. Full protocol and simulation details are in Appendices~\ref{sec:protocol} and~\ref{sec:sim}.

\begin{figure}
    \centering
\includegraphics[width=1\linewidth]{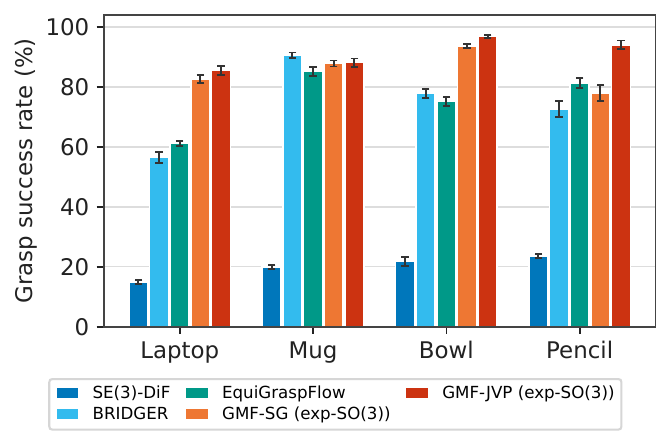}
    \caption{Per-category grasp success at NFE~5 with the exp-$\mathrm{SO}(3)$ sampler. Error bars denote the standard error of the mean over rotated test instances.}
\vspace{-1em}
\label{fig:succ_cat5}
\end{figure}

\subsection{Results}
\paragraph{Distributional fidelity.}
With a single function evaluation, GMF-SG obtains $0.418$ EMD and remains the best-performing method through NFE~5, demonstrating high distributional fidelity in the few-step regime; GMF-JVP reaches $0.466$ under the same budget (Figure~\ref{fig:emd_nfe}). Both are far below the baselines at NFE~1, where the closest competitor is at $0.694$, and EquiGraspFlow collapses to $0.734$. With additional steps, the two equivariant flow models stabilize. Both configurations therefore reach an accurate grasp distribution with substantially fewer function evaluations, and the coupling that makes GMF-JVP the stronger model for grasp success costs it roughly $0.05$ EMD throughout.

\paragraph{Grasp success.}
At $\mathrm{NFE}\,1, 2$ no method is usable: the best four-category average is $53.9$. The exp-$\mathrm{SO}(3)$ schedule is particularly ill-suited to this budget, since $\eta$ saturates at $1$ and the first update therefore commits the rotation to an endpoint predicted from a nearly untransported prior sample. We therefore focus on $\mathrm{NFE}\,5$ and above. At NFE~5, GMF-JVP is the strongest method on most categories, exceeding EquiGraspFlow by $21.7$, $24.3$, $3.0$, and $12.8$ points on Bowl, Laptop, Mug, and Pencil, while GMF-SG trails EquiGraspFlow only on Pencil, by $3.4$ points (Figure~\ref{fig:succ_cat5}). BRIDGER is marginally stronger on Mug ($90.6$ versus $88.1$), benefiting from its informative heuristic prior, but trails by $19.0$ and $29.1$ points on Bowl and Laptop. At NFE~10, GMF-JVP averages $96.3\%$ against EquiGraspFlow's $91.2\%$, and at NFE~20 the fully converged baseline takes the lead ($97.7\%$ versus $95.6\%$): the advantage of average-velocity modeling is concentrated in the intended few-step regime. Per-category results at every budget are in Table~\ref{tab:success}, and success under Euler integration is in Table~\ref{tab:posttrain} in the appendix.

\begin{figure}[t]
    \centering
\includegraphics[width=1\linewidth]{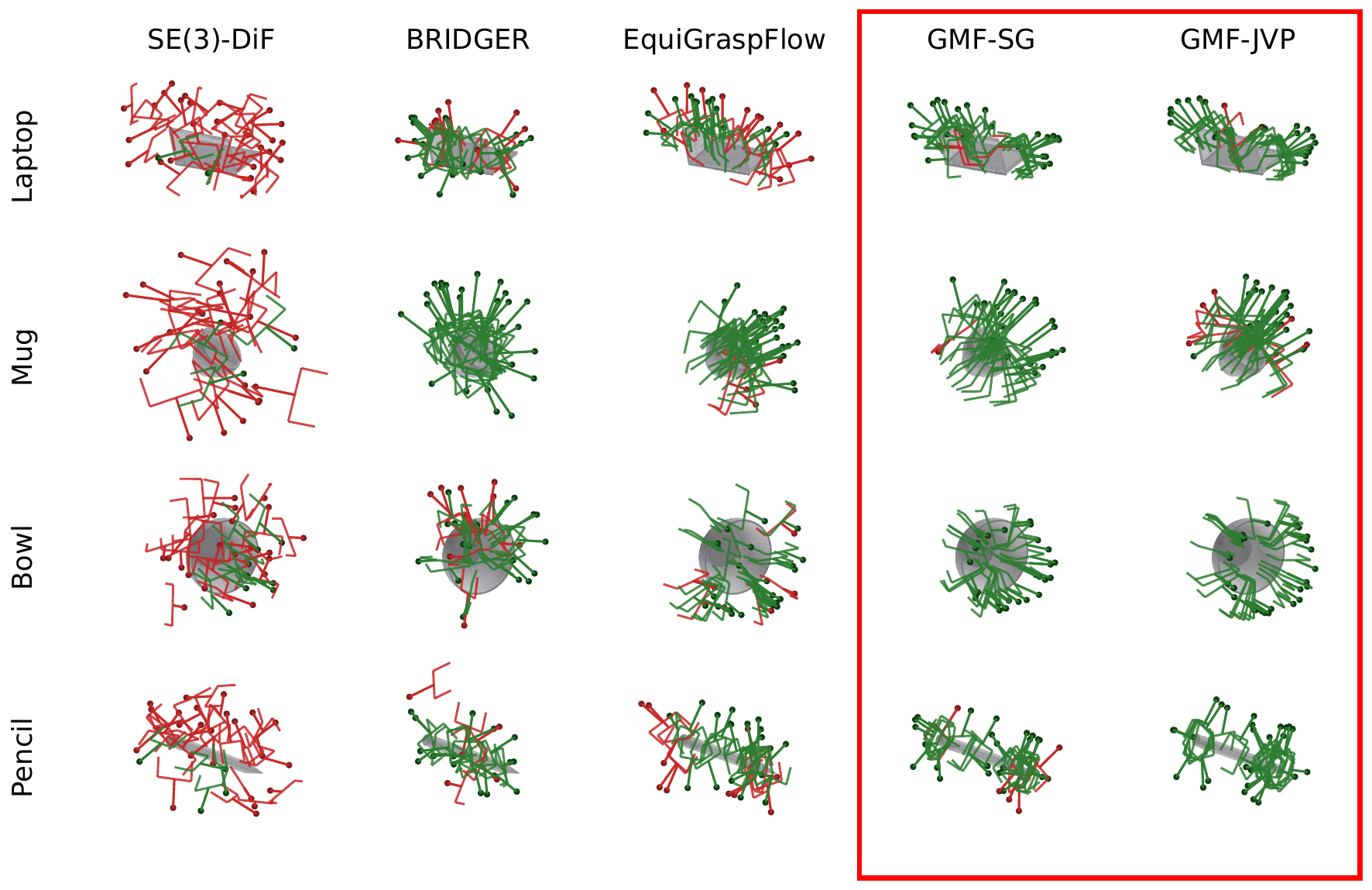}
    \caption{Grasps generated at NFE~5 for one test instance per category (rows), colored by simulated outcome (green: success, red: failure). Every column is generated from the same rotated object with the same protocol; the two GraspMeanFlow configurations (red box) use the exp-$\mathrm{SO}(3)$ sampler and the baselines their own samplers.}
    \vspace{-1em}
\label{fig:qual_rotcons}
\end{figure}

\paragraph{Ablation.}
Supervising the differential MeanFlow identity requires the spatial-frame formulation of Sec.~\ref{sec:se3-meanflow}: pairing the body-frame trajectory tangent with the left Jacobian, as an earlier implementation of ours did, diverges on the rotationally symmetric Bowl ($1.11$ one-step EMD, against $0.53$ for the boundary-plus-semigroup objective); with the frame-consistent tangent and the inverse-Jacobian target the same identity trains stably and yields GMF-JVP. Either consistency term must be paired with a data-anchored loss, since self-consistency alone admits collapsed minimizers (Appendix~\ref{sec:semigroup}). Complete ablations are in Appendix~\ref{sec:ablation}.

\paragraph{Inference cost.}
Latency should be compared at matched quality rather than at matched step
counts, since the baselines use smaller networks and are cheaper per
evaluation. GMF-SG reaches an EMD of $0.418$ in $31.5$\,ms with a single
evaluation, where EquiGraspFlow needs $44.7$\,ms and neither diffusion
baseline reaches that EMD at any budget; GMF-JVP reaches $96.3\%$ success
in $56.3$\,ms, where SE(3)-DiffusionFields needs $1338.7$\,ms for
$96.1\%$. Full latency curves are in Appendix~\ref{sec:latency}.

\vspace{-0.1em}

\paragraph{Training budget and model size.}
GraspMeanFlow adds no capacity over its flow-matching counterpart: the network has $633$k parameters, the EquiGraspFlow trunk with only a two-time embedding added, and both configurations are trained for 120k steps (${\sim}7$ hours on one A100). GMF-JVP additionally solves one assignment problem per object per step, which is negligible next to the forward and backward passes. Few-step performance therefore comes from the objective, not from model scale or training budget.

\section{Discussion and Limitations}

Our results locate the advantage of average-velocity modeling squarely in the
few-step regime. At NFE~5, GraspMeanFlow improves simulated success over
EquiGraspFlow by up to $24.3$ points, whereas at NFE~20 the iterative baseline
has integrated its instantaneous field accurately enough that little remains to
be gained, and the two are comparable on Bowl and Mug. This is the expected
trade-off: the average-velocity field absorbs the integration error that a
solver would otherwise accumulate, which matters most when the step budget is
small. The network, dataset, and pre-training budget are shared with
EquiGraspFlow, so the gain does not come from capacity. It comes from the objective together with the sampler the objective makes
available: an interval-averaged field admits an endpoint prediction at every
step, which an instantaneous field does not. Figure~\ref{fig:qual_rotcons} shows the generated grasps for one instance per category at NFE~5.

Two limitations are worth stating. First, Pencil is the category on which our gains are least consistent: GMF-SG trails EquiGraspFlow at every step budget, and GMF-JVP, while ahead at NFE~5 and~10, falls behind at NFE~20. We
attribute this to symmetry: a pencil is nearly invariant under rotations about
its long axis, so grasp orientations that are physically equivalent differ by a
rotation that our $\log$-displacement target treats as a genuine distance. A
symmetry-aware target---taking the logarithm to the nearest representative in
the object's symmetry group---would remove this mismatch, and we leave it to
future work. Second, how to schedule the endpoint rate as a function of the step budget
remains an open question. We also report a negative result: post-training on success-filtered self-generated samples helps the Euler sampler but degrades the endpoint-style sampler (Table~\ref{tab:posttrain}).

\section{Conclusion}
We introduced GraspMeanFlow, an $\mathrm{SE}(3)$-equivariant few-step generative
model for 6-DoF grasp poses, in which the average velocity is defined through
the time-ordered exponential and therefore reproduces exactly the rigid-body
displacement accumulated over an interval. We proved that a
point-cloud-conditioned distribution transported by an equivariant
average-velocity flow map remains invariant, so the equivariance guarantee
survives the reduction in sampling steps, at no architectural cost beyond a
two-scalar time embedding. Two instantiations sharing this network prove
complementary on ACRONYM, one on distributional fidelity and the other on
simulated success. A symmetry-aware training target and a step-budget-aware
sampler schedule remain open.

\section*{Acknowledgments}
We gratefully acknowledge support from the National Science Foundation
(DMS-2533878, DMS-2053746, DMS-2134209, ECCS-2328241, CBET-2347401,
and OAC-2311848) and the U.S.\ Department of Energy (DOE) Office of
Science, Advanced Scientific Computing Research program, under the
``Uncertainty Quantification for Multifidelity Operator Learning
(MOLUcQ)'' project (Project No.~81739), DE-SC0023161, the SciDAC LEADS
Institute, and DOE Fusion Energy Sciences under Grant No.~DE-SC0024583.

\bibliography{refs}
\appendix
\onecolumn

\makeatletter
\def\addcontentsline#1#2#3{%
  \addtocontents{#1}{\protect\contentsline{#2}{#3}{\thepage}{}}}
\makeatother
\etocsettocstyle{\begin{center}\section*{Appendix Contents}\end{center}}{}
\etocsetnexttocdepth{section}
\tableofcontents

\newpage

\section{Background: SO(3) Space}
\label{sec:so3-background}
\subsection{Basic Concepts in SO(3) Space}
    
The special orthogonal group in three dimensions is
$$
\mathrm{SO}(3):=\{ R \in \mathbb{R}^{3\times3} : R R^\top = R^\top R = I_3,\ \det(R)=1\}.
$$
The constraint $R R^\top = I_3$ consists of smooth polynomial equations, so
$\mathrm{SO}(3)$ is a smooth embedded submanifold of $\mathbb{R}^9$. Equipped with
matrix multiplication $(S_1,S_2)\mapsto S_1S_2$ it forms a Lie group, with
identity $I_3$ and inverse $R^{-1}=R^\top$.

\paragraph{Tangent space.}
Let $R\in\mathrm{SO}(3)$ be fixed and let $R(t)$ be a smooth curve with $R(0)=R$.
Differentiating $R(t)R(t)^\top=I_3$ at $t=0$ yields
$$
\dot R(0) R^\top + R \dot R(0)^\top = 0 ,
$$
so $\dot R(0)R^\top$ is skew-symmetric. The tangent space at $R$ is therefore
$$
T_R\mathrm{SO}(3)=\{\,[\omega]\,R : \omega\in\mathbb{R}^3\,\},
$$
where the Lie algebra
$$
\mathfrak{so}(3):=\{\Omega\in\mathbb{R}^{3\times3}:\Omega^\top=-\Omega\}
$$
is identified with $\mathbb{R}^3$ through the hat map $[\cdot]$ and its inverse
$(\cdot)^\vee$ defined in Section~\ref{sec:se3-basics}. This gives a global linear
parameterization of tangent vectors, in which $\omega$ is the angular velocity
measured in the world frame.

\paragraph{Lie algebra and Lie bracket.}
The Lie bracket on $\mathfrak{so}(3)$ is the matrix commutator
$$
\big[[\omega_1],[\omega_2]\big]=[\omega_1][\omega_2]-[\omega_2][\omega_1],
$$
which preserves skew-symmetry, so $\mathfrak{so}(3)$ is closed under it. In
vector coordinates it is the cross product,
$\big[[\omega_1],[\omega_2]\big]=[\,\omega_1\times\omega_2\,]$.

\subsection{Curves and Flows in SO(3)}

In Euclidean space a curve starting from $x_0$ is defined by the ODE $\dot X_t=v(t,X_t)$,
giving $X_1=x_0+\int_0^1 v(t,x_t)\,dt$. We now extend this construction to the
Lie group $\mathrm{SO}(3)$.

\paragraph{ODE-defined curves on $\mathrm{SO}(3)$.}
Let $R_t:[0,1]\to\mathrm{SO}(3)$ be a time-dependent curve. Its evolution is
given by
$$
\begin{cases}
R_0=r_0,\\
\dot R_t=[\omega_t]\,R_t, \quad \omega_t\in\mathbb{R}^3 .
\end{cases}
$$
This guarantees that $R_t$ remains in $\mathrm{SO}(3)$ for all $t$, since
$$
\frac{d}{dt}\big(R_tR_t^\top\big)
=[\omega_t]R_tR_t^\top+R_tR_t^\top[\omega_t]^\top
=[\omega_t]+[\omega_t]^\top=0 .
$$

\paragraph{Exponential map and geodesics.}
When the angular velocity is constant, $\omega_t=\omega$, the ODE admits the
closed-form solution
$$
R_t=\exp\big(t[\omega]\big)\,R_0 ,
$$
and under the canonical bi-invariant Riemannian metric on $\mathrm{SO}(3)$,
curves of this form are geodesics. Given two points
$R_0,R_1\in\mathrm{SO}(3)$, requiring $R_1=\exp([\omega])R_0$ gives
$\exp([\omega])=R_1R_0^\top$, so the geodesic connecting them is
$$
R_t=\exp\big(t\log(R_1R_0^\top)\big)\,R_0 ,\qquad t\in[0,1],
$$
where $\log$ denotes the matrix logarithm mapping $\mathrm{SO}(3)$ to
$\mathfrak{so}(3)$. This curve minimizes path length among all smooth curves
on $\mathrm{SO}(3)$ connecting $R_0$ and $R_1$.

\paragraph{Inner product and geodesic distance in $\mathrm{SO}(3)$.}
Under the canonical bi-invariant Riemannian metric we identify
$T_R\mathrm{SO}(3)=\{[\omega]R\}$ and define the inner product
$$
\big\langle [\omega_1]R,\,[\omega_2]R\big\rangle_R
=\tfrac12\,\mathrm{tr}\big([\omega_1]^\top[\omega_2]\big)
=\omega_1^\top\omega_2 ,
$$
which is invariant under both left and right multiplication. The induced
geodesic distance between $R_0,R_1\in\mathrm{SO}(3)$ is
$$
\mathrm{dist}(R_0,R_1)
=\frac{1}{\sqrt{2}}\,\big\|\log(R_1R_0^\top)\big\|_F
=\big\|\log(R_1R_0^\top)^\vee\big\|_2 ,
$$
where $\|\cdot\|_F$ denotes the Frobenius norm; this is the rotation angle
between the two frames.

\paragraph{Integration in SO(3).}
In Euclidean space, $x_1=x_0+\int_0^1 v(t,x_t)\,dt$. Due to the non-commutative
group structure, the endpoint in $\mathrm{SO}(3)$ cannot be written as a simple
integral. Instead, the solution at $t=1$ admits the group-valued representation
$$
R(1)=\mathcal{T}\!\exp\Big(\int_0^1 [\omega_t]\,dt\Big)\,R(0),
$$
where $\mathcal{T}$ denotes the time-ordering operator
$$
\mathcal{T}\!\exp\Big(\int_0^1 A(t)\,dt\Big)
=I_3+\sum_{k=1}^{\infty}
\int_{0\le t_k\le\cdots\le t_1\le 1}
A(t_1)\cdots A(t_k)\,dt_1\cdots dt_k ,
$$
which orders the matrix products chronologically, with the latest time on the
left. More generally, for $s\le t$,
$$
\mathcal{T}\!\exp\Big(\int_s^t [\omega_\tau]\,d\tau\Big)=R_tR_s^\top ,
$$
the quantity denoted $D(s,t)$ in Section~\ref{sec:avg-velocity}.

\section{Background: MeanFlow in Euclidean Space}
\label{sec:meanflow-bg-details}

This appendix restates the MeanFlow framework of
\citet{geng2025meanflow} in Euclidean space, using the time convention of the
main text ($t=0$ data, $t=1$ prior, $s<t$). It is included for completeness; the
$\mathrm{SE}(3)$ formulation of Section~\ref{sec:se3-meanflow} specializes to
this construction on the flat translation branch.

\subsection{Average Velocity and the MeanFlow Identity}

Flow matching learns an instantaneous velocity field. Given data
$x\sim p_{\mathrm{data}}$ and prior $\epsilon\sim p_{\mathrm{prior}}$, the linear
path $z_t=(1-t)x+t\epsilon$ has conditional velocity $v_t=\epsilon-x$, and a
network $v_\theta(z_t,t)$ is trained to regress it; sampling integrates
$\dot z_\tau=v_\theta(z_\tau,\tau)$ from $t=1$ to $t=0$.

MeanFlow instead models the \emph{average velocity} over an interval $[s,t]$,
\begin{equation}
    u(s,t,z_t)\;\triangleq\;\frac{1}{t-s}\int_s^t v(z_\tau,\tau)\,d\tau,
    \label{eq:mf-avg-def}
\end{equation}
a field jointly conditioned on the pair $(s,t)$. The displacement form
$(t-s)\,u(s,t,z_t)=\int_s^t v\,d\tau$ makes two properties immediate: the
diagonal limit
\begin{equation}
    \lim_{s\to t}u(s,t,z_t)=v(z_t,t),
    \label{eq:mf-boundary}
\end{equation}
and, from additivity of the integral, the interval-consistency relation
\begin{equation}
    (t-s)\,u(s,t,z_t)=(m-s)\,u(z_m,s,m)+(t-m)\,u(z_t,m,t),
    \qquad s\le m\le t.
    \label{eq:mf-consistency}
\end{equation}
Differentiating the displacement form with respect to $t$ (with $s$ held fixed) and using the
fundamental theorem of calculus yields the \emph{MeanFlow identity}, which
expresses the average velocity through the instantaneous velocity and a
trajectory derivative:
\begin{equation}
    u(s,t,z_t)\;=\;v(z_t,t)\;-\;(t-s)\,\frac{d}{dt}u(s,t,z_t).
    \label{eq:mf-identity-app}
\end{equation}

\subsection{Training Objective}

The total derivative in \eqref{eq:mf-identity-app} expands, via the chain rule
and $\dot z_t=v(z_t,t)$, $\dot s=0$, $\dot t=1$, into a Jacobian--vector product
(JVP):
\begin{equation}
    \frac{d}{dt}u(s,t,z_t)
    \;=\;v(z_t,t)\,\partial_z u\;+\;\partial_t u,
    \label{eq:mf-jvp-app}
\end{equation}
computable in a single forward-mode pass (e.g.\ \texttt{torch.func.jvp}).
Parameterizing $u_\theta$ and using the conditional velocity $v_t=\epsilon-x$ as
the only ground-truth signal, the objective regresses $u_\theta$ against a
stop-gradient target,
\begin{equation}
    \mathcal{L}(\theta)=\mathbb{E}\big\|u_\theta(s,t,z_t)-\mathrm{sg}(u_{\mathrm{tgt}})\big\|_2^2,
    \qquad
    u_{\mathrm{tgt}}=v_t-(t-s)\big(v_t\,\partial_z u_\theta+\partial_t u_\theta\big).
    \label{eq:mf-loss-app}
\end{equation}
The stop-gradient avoids double backpropagation through the JVP. When $s=t$ the
correction term vanishes and \eqref{eq:mf-loss-app} reduces exactly to flow
matching; in practice a fixed fraction of each batch is drawn with $s=t$ to
anchor the diagonal boundary \eqref{eq:mf-boundary}.

\subsection{Sampling}

Because $u_\theta$ already represents the interval-averaged field, sampling
replaces the integral by a direct average-velocity step,
\begin{equation}
    z_s=z_t-(t-s)\,u_\theta(s,t,z_t),
    \label{eq:mf-sample-app}
\end{equation}
and one-step generation is $z_0=z_1-u_\theta(z_1,0,1)$ with
$z_1=\epsilon\sim p_{\mathrm{prior}}$; few-step sampling applies
\eqref{eq:mf-sample-app} on a schedule $1=t_K>\cdots>t_0=0$.

\subsection{Conditioning on the Time Pair}

The field $u_\theta(z,s,t)$ is conditioned on two time variables. Following
\citet{geng2025meanflow}, each of $t$ and the interval $t-s$ is passed through a
positional embedding and a small MLP, and the two are combined into the network
conditioning; the JVP in \eqref{eq:mf-jvp-app} is always taken with respect to
$u_\theta(\cdot,s,t)$ regardless of the embedding used. In the Euclidean
implementation this conditioning is injected through adaptive layer
normalization (adaLN). We depart from this choice in the equivariant setting, where an additive normalization shift would break
$\mathrm{SO}(3)$-equivariance.

\subsection{Classifier-Free Guidance at a Single Evaluation}

MeanFlow folds classifier-free guidance into the target field so that guided
sampling retains its nominal number of function evaluations, rather than
combining two evaluations per step. For a condition $c$ and guidance weight
$\beta$, the guided instantaneous field is the affine combination
\begin{equation}
    v^{\mathrm{cfg}}(z_t,t\mid c)=\beta\,v(z_t,t\mid c)+(1-\beta)\,v(z_t,t),
    \label{eq:mf-cfg-v}
\end{equation}
and the network directly models the corresponding \emph{guided average
velocity} $u^{\mathrm{cfg}}_\theta$. Using the diagonal identity
$v^{\mathrm{cfg}}(z_t,t)=u^{\mathrm{cfg}}(z_t,t,t)$, the training target takes the
same form as \eqref{eq:mf-loss-app} with $v_t$ replaced by the guided data
velocity
\begin{equation}
    \tilde v_t=\beta\,v_t+(1-\beta)\,u^{\mathrm{cfg}}_\theta(z_t,t,t),
    \label{eq:mf-cfg-target}
\end{equation}
and the condition $c$ is dropped with a fixed probability during training.
Because $u^{\mathrm{cfg}}_\theta$ already models the guided field, sampling uses
it directly in \eqref{eq:mf-sample-app}: no combination is formed at inference,
so guidance incurs no additional function evaluations. Setting $\beta=1$
recovers the unguided objective \eqref{eq:mf-loss-app}.

\section{Detailed Formulation of the SE(3) MeanFlow Method}
\subsection{Model and Loss Function}
\label{sec:se3-mf-details}

We first collect the differential-geometric facts used below.

\begin{remark}\label{rmk:calculus-facts}
\begin{itemize}
\item \textbf{The hat map does not contribute to the derivative.}
Let $\omega:[0,1]\to\mathbb{R}^3$ be a smooth curve. Since $[\cdot]$ is linear
and independent of $t$, the chain rule gives
$$
    \frac{d}{dt}[\omega(t)]=\Big[\frac{d\omega}{dt}\Big]\in\mathbb{R}^{3\times 3}.
$$

\item \textbf{Fr\'echet derivative of the matrix exponential.}
Let $M(t):[0,1]\to\mathbb{R}^{3\times 3}$ be a smooth matrix curve and
$F(t)=\exp(M(t))$. By the chain rule on Banach spaces,
$\dot F(t)=d(\exp)_{M(t)}[\dot M(t)]$, where $d(\exp)_M$ admits the integral
(Wilcox) representation
$$
    d(\exp)_M(H)=\int_0^1 e^{(1-\alpha)M}\,H\,e^{\alpha M}\,d\alpha .
$$

\item \textbf{Conjugation identity.}
For any $Q\in\mathrm{SO}(3)$ and $\omega\in\mathbb{R}^3$,
$Q\,[\omega]\,Q^\top=[Q\omega]$.

\item \textbf{Derivative with respect to a point on a manifold.}
For a smooth curve $t\mapsto R(t)$ on $\mathrm{SO}(3)$ and a smooth
$f:\mathrm{SO}(3)\to\mathbb{R}$,
$$
    \tfrac{d}{dt}f(R(t))=df_{R(t)}[\dot R(t)]
    =\big\langle \nabla^{\mathcal M}_R f(R(t)),\,\dot R(t)\big\rangle ,
$$
with $\nabla^{\mathcal M}_R f$ the Riemannian gradient. Under the canonical
inner product $\langle A,B\rangle=\mathrm{Tr}(A^\top B)$ on
$\mathrm{SO}(3)\subset\mathbb{R}^{3\times3}$, the Riemannian gradient coincides
with the Euclidean gradient of any smooth extension.
\end{itemize}
\end{remark}

Recall the definition of the average velocity from \eqref{eq:se3-mean-velocity},
which we restate here in the notation of this appendix:
\begin{align}
\exp\big([\omega^{s\to t}]\big)
:=\mathcal{T}\!\exp\Big(\int_s^t [\omega(\tau,R_\tau,x_\tau)]\,d\tau\Big),
\qquad
\omega^{s\to t}=(t-s)\,\omega^{\mathrm{avg}}(s,t,R_t,x_t).
\label{eq:avg_omega}
\end{align}

\begin{remark}
$\mathrm{SO}(3)$ is not commutative, so in general
$\mathcal{T}\!\exp\big(\int_s^t[\omega_\tau]d\tau\big)\neq
\exp\big(\int_s^t[\omega_\tau]d\tau\big)$ unless $\omega_\tau$ is constant.
The average velocity therefore cannot be defined by integrating $\omega_\tau$
directly.
\end{remark}

Since each $[\omega(\tau,R_\tau,x_\tau)]\in\mathfrak{so}(3)$, we have
$\omega^{\mathrm{avg}}(s,t,R_t,x_t)\in\mathbb{R}^3$. In the extreme case
$s=0$, $t=1$, the average velocity transports the data sample to the prior sample and back:

$$
\begin{cases}
R_1=\exp\big([\omega^{0\to1}]\big)R_0 &\text{(forward)},\\[2pt]
R_0=\exp\big(-[\omega^{0\to1}]\big)R_1 &\text{(backward)}.
\end{cases}
$$

\begin{proposition}[Derivative of the averaged angular velocity]
\label{pro:Omega_identity}
Treating $s$ and $t$ as independent and writing
$\omega^{s\to t}=(t-s)\,\omega^{\mathrm{avg}}(s,t,R_t,x_t)$, differentiating both
sides of \eqref{eq:avg_omega} with respect to $t$ yields
\begin{align}
\begin{cases}
\text{L.H.S.}&=\Big[\,J\big(\omega^{s\to t}\big)\,
   \tfrac{d}{dt}\omega^{s\to t}\Big]\;R_tR_s^\top,\\[4pt]
\text{R.H.S.}&=\big[\omega(t,R_t,x_t)\big]\;R_tR_s^\top,
\end{cases}
\label{eq:omega_avg_deriv}
\end{align}
and hence $J(\omega^{s\to t})\,\tfrac{d}{dt}\omega^{s\to t}=\omega(t,R_t,x_t)$,
which is \eqref{eq:omega_identity}.
\end{proposition}

\begin{proof}
Differentiate both sides of \eqref{eq:avg_omega} with respect to $t$.

For the right-hand side, the derivative rule for the time-ordered exponential
under the spatial convention $\dot R_\tau=[\omega_\tau]R_\tau$ places the
generator on the \emph{left}:
\begin{align}
\frac{d}{dt}\,\mathcal{T}\!\exp\Big(\int_s^t[\omega_\tau]\,d\tau\Big)
&=\big[\omega(t,R_t,x_t)\big]\;
  \mathcal{T}\!\exp\Big(\int_s^t[\omega_\tau]\,d\tau\Big)\nonumber\\
&=\big[\omega(t,R_t,x_t)\big]\,\exp\big([\omega^{s\to t}]\big),
\label{eq:rhs_omega_avg}
\end{align}
where $\exp([\omega^{s\to t}])=R_tR_s^\top$ by definition.

For the left-hand side, the Wilcox formula, with the exponential factored out on
the left, gives
\begin{align}
\frac{d}{dt}\exp\big([\omega^{s\to t}]\big)
&=\int_0^1 e^{(1-\alpha)[\omega^{s\to t}]}
  \Big[\tfrac{d}{dt}\omega^{s\to t}\Big]
  e^{\alpha[\omega^{s\to t}]}\,d\alpha\nonumber\\
&=\bigg(\int_0^1 e^{\alpha[\omega^{s\to t}]}
  \Big[\tfrac{d}{dt}\omega^{s\to t}\Big]
  e^{-\alpha[\omega^{s\to t}]}\,d\alpha\bigg)\,
  \exp\big([\omega^{s\to t}]\big)\nonumber\\
&=\bigg(\int_0^1
  \Big[e^{\alpha[\omega^{s\to t}]}\,\tfrac{d}{dt}\omega^{s\to t}\Big]
  \,d\alpha\bigg)\,\exp\big([\omega^{s\to t}]\big)\nonumber\\
&=\Big[\Big(\int_0^1 e^{\alpha[\omega^{s\to t}]}d\alpha\Big)
  \tfrac{d}{dt}\omega^{s\to t}\Big]\;
  \exp\big([\omega^{s\to t}]\big).
\label{eq:lhs_omega_avg}
\end{align}
The second line substitutes $\alpha\mapsto1-\alpha$ and factors
$\exp([\omega^{s\to t}])$ to the right; the third uses the conjugation
identity $Q[a]Q^\top=[Qa]$ with $Q=e^{\alpha[\omega^{s\to t}]}\in\mathrm{SO}(3)$;
the last uses linearity of $[\cdot]$ in $\alpha$. Recognizing
\begin{align}
J\big(\omega^{s\to t}\big)=\int_0^1 e^{\alpha[\omega^{s\to t}]}\,d\alpha
\nonumber
\end{align}
as the left Jacobian \eqref{eq:left_jacobian} and equating
\eqref{eq:rhs_omega_avg} with \eqref{eq:lhs_omega_avg} gives
\eqref{eq:omega_avg_deriv}; canceling the common right factor
$\exp([\omega^{s\to t}])$ and applying $(\cdot)^\vee$ completes the
proof.
\end{proof}

\begin{remark}[Closed form of the left Jacobian]
\label{rmk:jacobian-closed}
Evaluating the integral gives, for $\phi\in\mathbb{R}^3$ with
$\|\phi\|\neq0$,
$$
J(\phi)=I+\frac{1-\cos\|\phi\|}{\|\phi\|^{2}}\,[\phi]
       +\frac{\|\phi\|-\sin\|\phi\|}{\|\phi\|^{3}}\,[\phi]^{2},
$$
with $J(\phi)\to I$ as $\|\phi\|\to0$; in practice we use the Taylor expansion
of this expression for numerical stability in that regime.
\end{remark}

\begin{remark}[Constant angular velocity]
If the angular velocity is constant, $\omega_\tau=\omega$, then
$\omega^{s\to t}=(t-s)\omega$ and $\tfrac{d}{dt}\omega^{s\to t}=\omega$. Since
$[\omega]\,\omega=\omega\times\omega=0$, every term of the series
\eqref{eq:left_jacobian} beyond $k=0$ annihilates $\omega$, so
$J(\omega^{s\to t})\,\omega=\omega$: the Jacobian does not distort the velocity
when the rotation axis is fixed.
\end{remark}

\begin{proposition}\label{pro:dA/dt}
The total derivative of $\omega^{\mathrm{avg}}$ along the trajectory is
\begin{align}
\frac{d}{dt}\omega^{\mathrm{avg}}(s,t,R_t,x_t)
&=\partial_t \omega^{\mathrm{avg}}
+\big\langle\nabla_R \omega^{\mathrm{avg}},\dot R_t\big\rangle
+\big\langle\nabla_x \omega^{\mathrm{avg}},\dot x_t\big\rangle .
\nonumber
\end{align}
\end{proposition}
\begin{proof}
By Remark~\ref{rmk:calculus-facts}, $d_R\omega^{\mathrm{avg}}[\,\cdot\,]$ agrees
with the Euclidean differential on $T_{R_t}\mathrm{SO}(3)$, and similarly for the
Euclidean component. Applying the chain rule along $t\mapsto(R_t,x_t)$ gives the
claim.
\end{proof}

Substituting the network output $\omega^{\mathrm{avg}}_\theta(s,t,R_t,x_t)$ and
writing $\omega^{s\to t}_\theta=(t-s)\,\omega^{\mathrm{avg}}_\theta$, the
differential objective of \eqref{eq:diff-loss} reads
\begin{align}
\mathcal{L}^{\mathrm{diff}}
:=\mathbb{E}_{s<t,\;((R_0,x_0),(R_1,x_1))\sim q_{0,1}}
\Big[\big\|J\big(\omega^{s\to t}_\theta\big)\,
\tfrac{d}{dt}\omega^{s\to t}_\theta-\omega(t,R_t,x_t)\big\|_2^2\Big],
\label{eq:loss_so3}
\end{align}
where $q_{0,1}$ is a coupling between the data and prior distributions. The
interpolation and its velocity are
$$
\omega(t)=\log(R_1R_0^\top)^\vee,\quad
R_t=\exp\big(t[\omega(t)]\big)R_0,\quad
\dot R_t=[\omega(t)]\,R_t .
$$

\begin{proposition}\label{pro:loss_zero}
If $\mathcal{L}^{\mathrm{diff}}=0$, i.e.,
$J(\omega^{s\to t}_\theta)\,\tfrac{d}{dt}\omega^{s\to t}_\theta
=\omega(t,R_t,x_t)$ for all $s<t\in[0,1]$, then
$\exp\big([\omega^{s\to t}_\theta]\big)=R_tR_s^\top$.
\end{proposition}

\begin{proof}
Take $s=0$ for simplicity and define the candidate reconstruction
$\tilde R_t:=\exp\big([\omega^{0\to t}_\theta]\big)R_0$. By the
Fr\'echet derivative computed in \eqref{eq:lhs_omega_avg},
$$
\frac{d}{dt}\exp\big([\omega^{0\to t}_\theta]\big)
=\Big[J\big(\omega^{0\to t}_\theta\big)\tfrac{d}{dt}\omega^{0\to t}_\theta\Big]
 \,\exp\big([\omega^{0\to t}_\theta]\big),
$$
so the zero-loss condition gives
$\dot{\tilde R}_t=[\omega(t,R_t,x_t)]\,\tilde R_t$. The interpolation satisfies
the same equation, $\dot R_t=[\omega(t,R_t,x_t)]R_t$, with the same initial
condition $\tilde R_0=R_0$. By uniqueness of solutions on $\mathrm{SO}(3)$,
$\tilde R_t=R_t$ for all $t$, i.e.,
$\exp([\omega^{0\to t}_\theta])=R_tR_0^\top$. The same argument for
general $s<t$ gives $\exp([\omega^{s\to t}_\theta])=R_tR_s^\top$.
\end{proof}

Finally, expanding $\tfrac{d}{dt}\omega^{s\to t}_\theta$ by the product rule and
Proposition~\ref{pro:dA/dt},
\begin{align}
\frac{d}{dt}\omega^{s\to t}_\theta
&=\omega^{\mathrm{avg}}_\theta
+(t-s)\Big(\partial_t \omega^{\mathrm{avg}}_\theta
+\big\langle\nabla_R \omega^{\mathrm{avg}}_\theta,\dot R_t\big\rangle
+\big\langle\nabla_x \omega^{\mathrm{avg}}_\theta,\dot x_t\big\rangle\Big),
\label{eq:total_dOmega_dt}
\end{align}
which is the quantity computed as a single Jacobian--vector product in
\eqref{eq:diff-loss}.

\begin{algorithm}[h]
\caption{Differential SE(3) MeanFlow training step (not used in our final model)}
\label{alg:meanflow_loss}
\begin{algorithmic}[1]
\Require dataset $\mathcal{D}$, encoder $f_{\mathrm{enc}}$, two-time field
  $f_\theta$, dropout $p_{\mathrm{uncond}}$, diagonal probability $p_{\mathrm{diag}}$
\While{not converged}
  \State Sample $(\mathcal{P}, R_0, x_0) \sim \mathcal{D}$;\;\;
    $z \gets f_{\mathrm{enc}}(\mathcal{P})$;\;\;
    $z \gets 0$ w.p.\ $p_{\mathrm{uncond}}$
  \State Sample $R_1 \sim \mathrm{Unif}(\mathrm{SO}(3))$,\;
    $x_1 \sim \mathcal{N}(0, I)$,\;
    $t \sim \mathrm{Unif}(0,1]$,\;
    $s \sim \mathrm{Unif}(0,t)$
    \Comment{$s \gets t$ w.p.\ $p_{\mathrm{diag}}$}
  \State $\omega_t \gets \log(R_1 R_0^\top)^{\vee}$;\;\;
    $v_t \gets x_1 - x_0$
  \State $R_t \gets \exp\!\big(t[\omega_t]\big) R_0$;\;\;
    $x_t \gets (1{-}t)x_0 + t x_1$
  \State $\dot s \gets 0$;\;\; $\dot t \gets 1$;\;\;
    $\dot R_t \gets [\omega_t] R_t$;\;\; $\dot x_t \gets v_t$
    \Comment{trajectory tangent}
  \State $\big(\omega^{\mathrm{avg}}_\theta, v^{\mathrm{avg}}_\theta\big),\;
    \big(\dot\omega^{\mathrm{avg}}_\theta, \dot v^{\mathrm{avg}}_\theta\big)
    \gets \mathrm{jvp}\big(f_\theta;\;(s,t,R_t,x_t),\;
    (\dot R_t,\dot x_t,\dot s,\dot t)\big)$
    \Comment{one forward-mode pass, both branches}
  \State $\omega^{s\to t}_\theta \gets (t-s)\,\omega^{\mathrm{avg}}_\theta$
  \State $\omega^{\mathrm{avg}}_{\mathrm{tgt}} \gets
    J\big(\omega^{s\to t}_\theta\big)^{-1}\omega_t
    - (t-s)\,\dot\omega^{\mathrm{avg}}_\theta$
    \Comment{rotation: $J\neq I$}
  \State $v^{\mathrm{avg}}_{\mathrm{tgt}} \gets
    v_t - (t-s)\,\dot v^{\mathrm{avg}}_\theta$
    \Comment{translation: $J\equiv I$}
  \State $\mathcal{L}^{\mathrm{diff}} \gets
    \lVert \omega^{\mathrm{avg}}_\theta
    - \mathrm{sg}\big(\omega^{\mathrm{avg}}_{\mathrm{tgt}}\big)\rVert^2
    + \lVert v^{\mathrm{avg}}_\theta
    - \mathrm{sg}\big(v^{\mathrm{avg}}_{\mathrm{tgt}}\big)\rVert^2$
  \State Adam step on $\mathcal{L}^{\mathrm{diff}}$;\; EMA update
\EndWhile
\end{algorithmic}
\end{algorithm}

\begin{figure}
    \centering
    \includegraphics[width=0.5\linewidth]{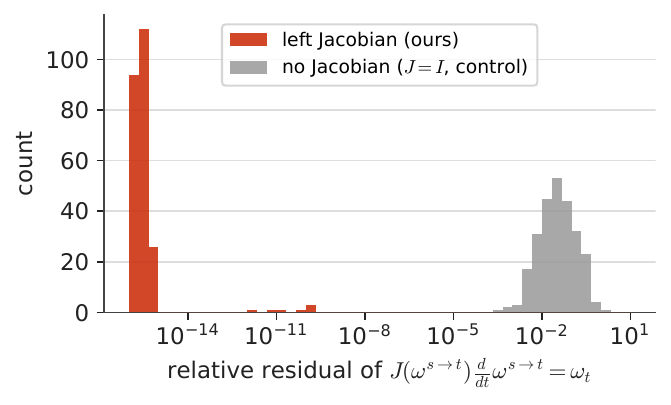}
    \caption{Relative residual of \eqref{eq:omega_identity} over $256$ random smooth trajectories, computed with the left Jacobian (red) and with the Jacobian dropped ($J{=}I$, gray). Thirteen orders of magnitude separate the exact identity from its linearization.} \label{fig:id16-check}
\end{figure}

\paragraph{Numerical validation of \eqref{eq:omega_identity}.} We evaluate the identity on closed-form, non-geodesic trajectories $R_t=\exp([\phi(t)])$ with $\phi$ random smooth curves, for which the spatial velocity $\omega_t=J(\phi)\dot\phi$ is available in closed form; $\omega^{s\to t}=\log(R_tR_s^{\top})^{\vee}$ and the total derivative \eqref{eq:domega-dt} are computed by automatic differentiation, so the state dependence along the trajectory is treated exactly. In double precision, the residual is at most $2\times10^{-10}$ and concentrates at $10^{-15}$ (Figure~\ref{fig:id16-check}); the small tail stems from the numerical guards of the matrix logarithm near $\pi$. Dropping the Jacobian instead leaves residuals of order one, so the identity is genuinely the $J$-weighted relation and not its small-rotation linearization.

\section{JVP-Free Training via SE(3) $\alpha$-Flow}
\label{sec:alpha-flow}

The differential objective \eqref{eq:diff-loss} requires the trajectory
derivative $\tfrac{d}{dt}\omega^{s\to t}_\theta$ through a Jacobian--vector
product of $f_\theta$, composed with the left Jacobian $J$. We adopt the
$\alpha$-Flow framework of \citet{cheng2025alpha}, which replaces this
differential target with a two-evaluation, JVP-free consistency target, and we
develop it on $\mathrm{SE}(3)$. Throughout, the model outputs the average
velocity pair $(\omega^{\mathrm{avg}}_\theta,v^{\mathrm{avg}}_\theta)(s,t,T_t)$,
and we write $\omega^{s\to t}_\theta=(t-s)\omega^{\mathrm{avg}}_\theta$ and
$p^{s\to t}_\theta=(t-s)v^{\mathrm{avg}}_\theta$ for the rotation
log-displacement and the translation displacement.

\subsection{Time Grid}
Let $\alpha\in[\alpha_{\min},1]$ be the consistency-step ratio, floored at a
small $\alpha_{\min}>0$. For $0\le s\le t\le1$ define the intermediate time and
step
\begin{align}
m=\alpha s+(1-\alpha)t,\qquad \delta=t-m=\alpha(t-s),\qquad s\le m\le t.
\nonumber
\end{align}
Write $\omega_t$ and $v_t$ for the data-side instantaneous velocities of the
interpolation \eqref{eq:geodesic_R}--\eqref{eq:geodesic_x}, and define the
stepped-back intermediate state
$$
R_m=\exp\big(-\delta[\omega_t]\big)R_t,
\qquad
x_m=x_t-\delta\,v_t .
$$

\subsection{Translation Target}
\label{sec:af-trans}
The average velocity over $[s,t]$ decomposes through the velocities over $[s,m]$
and $[m,t]$. For the near segment $[m,t]$ we use the shift velocity
$$
\tilde v_t=
\begin{cases}
v_t, & \text{(data velocity; flow-matching / MeanFlow choice)},\\[2pt]
v^{\mathrm{avg}}_\theta(m,t,T_t), & \text{(model prediction; Shortcut choice)},
\end{cases}
$$
and adopt the data velocity $\tilde v_t=v_t$. For the far segment $[s,m]$ we use
the stop-gradient model prediction at the intermediate state. The $\alpha$-Flow
target is the convex combination
\begin{align}
v^{\mathrm{avg}}_{\mathrm{tgt}}
=\alpha\,\tilde v_t+(1-\alpha)\,v^{\mathrm{avg}}_\theta(s,m,T_m),
\label{eq:af-trans-target}
\end{align}
equivalently, in displacement form,
$p^{s\to t}_{\mathrm{tgt}}=\delta\,\tilde v_t+(m-s)\,v^{\mathrm{avg}}_\theta(s,m,T_m)$.
The loss is
\begin{align}
\mathcal{L}^{\alpha}_{\mathrm{trans}}
=\frac{1}{\alpha}\,\mathbb{E}\Big[\big\|p^{s\to t}_\theta
-\mathrm{sg}\big(p^{s\to t}_{\mathrm{tgt}}\big)\big\|_2^2\Big].
\label{eq:af-trans-loss}
\end{align}

\subsection{Rotation Target}
\label{sec:af-rot}
We mirror the translation target on $\mathrm{SO}(3)$, replacing vector addition
by group composition. For $0\le a\le b\le1$ the accumulated relative rotation is
\begin{align}
D(a,b):=\mathcal{T}\!\exp\Big(\int_a^b[\omega_\tau]\,d\tau\Big)
=\exp\big([\omega^{a\to b}]\big)=R_bR_a^\top\in\mathrm{SO}(3),
\label{eq:af-Dab}
\end{align}
where the middle equality is the definition \eqref{eq:se3-mean-velocity} of
$\omega^{a\to b}$ and the last follows by integrating $\dot R_\tau=[\omega_\tau]R_\tau$;
both hold for an arbitrary velocity field, and $\omega^{a\to b}$ is unique on the
principal branch $\|\omega^{a\to b}\|\le\pi$. 

\begin{proposition}[Interval additivity]\label{prop:af-additivity}
For any $s\le m\le t$, $\;D(s,t)=D(m,t)\,D(s,m)$.
\end{proposition}
\begin{proof}
$D(s,t)=R_tR_s^\top=(R_tR_m^\top)(R_mR_s^\top)=D(m,t)D(s,m)$. The regrouping is
exact and independent of the geodesic assumption; it is the non-commutative
$\mathrm{SO}(3)$ analogue of Euclidean displacement additivity, with addition
replaced by group multiplication and the near segment $[m,t]$ ordered to the
left of the far segment $[s,m]$.
\end{proof}

Mirroring \eqref{eq:af-trans-target}, the near segment uses the shift angular
velocity
$$
\tilde\omega_t=
\begin{cases}
\omega_t, & \text{(data angular velocity)},\\[2pt]
\omega^{\mathrm{avg}}_\theta(m,t,T_t), & \text{(model prediction; Shortcut choice)},
\end{cases}
$$
and we adopt $\tilde\omega_t=\omega_t$, so that
$D(m,t)=\exp(\delta[\tilde\omega_t])$. The far segment uses the
stop-gradient model evaluation at the stepped-back state,
$$
\omega^{s\to m}_\theta=(m-s)\,\omega^{\mathrm{avg}}_\theta(s,m,T_m),
\qquad
D(s,m)=\exp\big([\omega^{s\to m}_\theta]\big).
$$
By Proposition~\ref{prop:af-additivity}, the composed rotation target is
\begin{align}
\boxed{\;
\omega^{s\to t}_{\mathrm{tgt}}
=\log\Big(
\underbrace{\exp\big(\delta[\tilde\omega_t]\big)}_{D(m,t)}\;
\underbrace{\exp\big([\omega^{s\to m}_\theta]\big)}_{D(s,m)}
\Big)^{\!\vee}.\;}
\label{eq:af-rot-target}
\end{align}
The two segment generators do not commute, so the composition must be formed as
$\log(\exp\cdot\exp)$ rather than collapsed into a Lie-algebra
sum: rescaling or adding the two segment angles would corrupt the
Baker--Campbell--Hausdorff (BCH) cross term and no longer produce $\log D(s,t)$.

The target requires one logarithm, two exponentials, and no differentiation of
$f_\theta$. Both segment angles are bounded --- $\|\omega^{s\to m}_\theta\|\le\pi$
by the principal branch of $\log$, and $\|\delta\tilde\omega_t\|\le\alpha(t-s)\pi$
--- so the target is bounded by $2\pi$ and free of small-denominator blow-up. For
$t-s$ below machine tolerance $t_\varepsilon$, where $\log(\cdot)$ near the
identity loses precision, we use the first-order BCH limit, whose $O(t-s)$
correction is then negligible:
\begin{align}
\omega^{s\to t}_{\mathrm{tgt}}=
\begin{cases}
\log\big(D(m,t)\,D(s,m)\big)^\vee, & t-s\ge t_\varepsilon,\\[6pt]
\delta\,\tilde\omega_t+\omega^{s\to m}_\theta, & t-s< t_\varepsilon.
\end{cases}\nonumber
\end{align}
The rotation loss is then
\begin{align}
\mathcal{L}^{\alpha}_{\mathrm{rot}}
=\frac{1}{\alpha}\,\mathbb{E}\Big[\big\|\omega^{s\to t}_\theta
-\mathrm{sg}\big(\omega^{s\to t}_{\mathrm{tgt}}\big)\big\|_2^2\Big],
\label{eq:af-rot-loss}
\end{align}
identical in form to \eqref{eq:af-trans-loss}.

\begin{remark}[Fallback at $\alpha=1$ and choice of normalization]\label{rmk:af-limits}
At $\alpha=1$ the intermediate time hits the far end, $m=s$ and $\delta=t-s$, so the
far segment vanishes ($D(s,m)=I$) and
$\omega^{s\to t}_{\mathrm{tgt}}=\log D(m,t)^\vee=(t-s)\,\tilde\omega_t$: the objective
reduces to flow matching \eqref{eq:bd-loss}, anchoring the branch to the data
velocity and supplying the boundary condition that prevents collapse. The opposite
limit $\alpha\to0$ recovers the differential MeanFlow objective and is analyzed in
Section~\ref{sec:af-meanflow-limit}.

The group composition $\log(\exp\cdot\exp)$ is essential throughout:
the naive Lie sum $\delta\,\tilde\omega_t+\omega^{s\to m}_\theta$ drops the
BCH/Jacobian correction and is not equivalent.
\end{remark}

\subsection{The $\alpha\to 0$ Limit Recovers MeanFlow}
\label{sec:af-meanflow-limit}

In this subsection we show that the $\alpha$-Flow loss converges to the differential
MeanFlow loss as $\alpha\to0$. We treat the rotation branch; the translation branch
is the abelian case ($J\equiv I$, no BCH correction) and reduces to Euclidean
MeanFlow \eqref{eq:mf-loss-app} verbatim.

\begin{remark}[$\alpha$-Flow normalization]\label{rmk:af-normalization}
If $\alpha$-Flow is to interpolate exactly between flow matching and MeanFlow, the
prefactor should be $1/\alpha^{2}$ rather than $1/\alpha$, since the residual
regressed in \eqref{eq:af-rot-loss} is $O(\alpha)$
(Proposition~\ref{prop:af-meanflow-limit}). We therefore analyze
\begin{align}
\widetilde{\mathcal{L}}_{\mathrm{rot}}
:=\frac{1}{\alpha^{2}}
\big\|\omega^{\mathrm{avg}}_{\theta}(s,t,T_t)
-\mathrm{sg}\big(\omega^{\mathrm{avg}}_{\mathrm{tgt}}\big)\big\|_2^2 ,
\label{eq:af-rot-loss-theory}
\end{align}
with $\omega^{\mathrm{avg}}_{\mathrm{tgt}}=\omega^{s\to t}_{\mathrm{tgt}}/(t-s)$ from
\eqref{eq:af-rot-target}; the $1/\alpha$ prefactor of \eqref{eq:af-rot-loss} follows
the style of \citet{cheng2025alpha} and is what we train with, because $1/\alpha^{2}$ is
less stable in practice. We also use only the first case of
\eqref{eq:af-rot-target}, the second being a small-$(t-s)$ fallback introduced for
numerical convenience.
\end{remark}

\begin{proposition}[$\alpha\to0$ recovers the MeanFlow loss]
\label{prop:af-meanflow-limit}
Let $g(\tau):=(\tau-s)\,\omega^{\mathrm{avg}}_\theta(s,\tau,T_\tau)$ denote the model
log-displacement along the interpolation path, so that
$g(t)=(t-s)\,\omega^{\mathrm{avg}}_\theta(s,t)$ and, by the product rule,
$\dot g(t)=\omega^{\mathrm{avg}}_\theta+(t-s)\tfrac{d}{dt}\omega^{\mathrm{avg}}_\theta$
is the quantity $\tfrac{d}{dt}\omega^{s\to t}$ appearing in
\eqref{eq:omega_identity}, and hence the quantity that
\eqref{eq:diff-target} regresses against. If $g$ is $C^{2}$ at $t$, then the
target \eqref{eq:af-rot-target} expands as
\begin{align}
\omega^{\mathrm{avg}}_{\mathrm{tgt}}
=\frac{g(t)}{t-s}+\alpha\Big(J\big(g(t)\big)^{-1}\omega_t-\dot g(t)\Big)+O(\alpha^{2})
\label{eq:af-tgt-expand}
\end{align}
and consequently
\begin{align}
\widetilde{\mathcal{L}}_{\mathrm{rot}}
=\big\|\dot g(t)-J\big(g(t)\big)^{-1}\omega_t\big\|_2^{2}+O(\alpha).
\label{eq:af-loss-limit}
\end{align}
The leading term is \emph{exactly} the per-sample rotation residual of the
differential MeanFlow loss \eqref{eq:diff-loss}: by \eqref{eq:diff-target} and the
definition of $g$,
$\omega^{\mathrm{avg}}_\theta-\omega^{\mathrm{avg}}_{\mathrm{tgt}}
=\dot g(t)-J(g(t))^{-1}\omega_t$. The $\alpha$-Flow objective therefore converges
to the differential MeanFlow loss itself, and does so without ever forming a
Jacobian--vector product.
\end{proposition}

\begin{proof}
Write $m=t-\delta$ with $\delta=\alpha(t-s)$, and $J:=J(g(t))$.

Since the data velocities $\omega_t$ and $v_t$ are constant along the interpolation
path (Section~\ref{sec:grasp-fm}), the stepped-back state satisfies
$\exp(-\delta[\omega_t])R_t=R_m$ and $x_t-\delta v_t=x_m$ \emph{exactly}. The
two factors of \eqref{eq:af-rot-target} are therefore
$D(s,m)=\exp\big([g(m)]\big)$ and $D(m,t)=\exp\big(\delta[\omega_t]\big)$,
with $g$ the same function appearing in the statement; this is what lets the finite
difference below capture the \emph{total} derivative \eqref{eq:domega-dt} rather than
$\partial_t$ alone.

The first-order BCH formula, expanded in the small \emph{left} factor, gives
$\log(e^{X}e^{Y})=Y+\frac{\mathrm{ad}_Y}{e^{\mathrm{ad}_Y}-1}X+O(\|X\|^{2})$. On
$\mathfrak{so}(3)$ one has $\mathrm{ad}_{[\phi]}[\psi]=\big[[\phi]\psi\big]$
(Remark~\ref{rmk:adjoint}), so in vector coordinates the coefficient is the matrix
$\big(e^{[\phi]}-I\big)^{-1}[\phi]=J(\phi)^{-1}$, using the closed form of
Remark~\ref{rmk:jacobian-closed}. This is invertible: $J(\phi)$ is a power series in
the skew matrix $[\phi]$, so its eigenvalues are $1$ and
$\tfrac{e^{\pm i\theta}-1}{\pm i\theta}$ with $\theta=\|\phi\|\le\pi$, all nonzero.
Taking $X=\delta[\omega_t]=O(\alpha)$ and $Y=[g(m)]$,
\begin{align}
\log\!\big(D(m,t)\,D(s,m)\big)^{\vee}
=g(m)+\delta\,J\big(g(m)\big)^{-1}\omega_t+O(\delta^{2}).
\nonumber
\end{align}
Substituting $g(m)=g(t)-\delta\dot g(t)+O(\delta^{2})$ and
$J(g(m))^{-1}=J^{-1}+O(\delta)$, we have
\begin{align}
\omega^{\mathrm{avg}}_{\mathrm{tgt}}
&=\frac{1}{t-s}\log\!\big(D(m,t)\,D(s,m)\big)^{\vee}
\nonumber\\
&=\frac{1}{t-s}\Big(g(m)+\delta\,J\big(g(m)\big)^{-1}\omega_t+O(\delta^{2})\Big)
\nonumber\\
&=\frac{1}{t-s}\Big(g(t)-\delta\,\dot g(t)
+\delta\,J^{-1}\omega_t+O(\delta^{2})\Big)
\nonumber\\
&=\frac{g(t)}{t-s}
+\frac{\alpha(t-s)}{t-s}\Big(J^{-1}\omega_t-\dot g(t)\Big)
+\frac{O\big(\alpha^{2}(t-s)^{2}\big)}{t-s}
\nonumber\\
&=\omega^{\mathrm{avg}}_\theta(s,t)
-\alpha\Big(\dot g(t)-J^{-1}\omega_t\Big)+O(\alpha^{2}),
\nonumber
\end{align}
which is \eqref{eq:af-tgt-expand}. Substituting into \eqref{eq:af-rot-loss-theory},
\begin{align}
\widetilde{\mathcal{L}}_{\mathrm{rot}}
&=\frac{1}{\alpha^{2}}
\Big\|\omega^{\mathrm{avg}}_\theta(s,t)
-\mathrm{sg}\big(\omega^{\mathrm{avg}}_{\mathrm{tgt}}\big)\Big\|_2^{2}
\nonumber\\
&=\frac{1}{\alpha^{2}}
\Big\|\omega^{\mathrm{avg}}_\theta(s,t)
-\Big(\omega^{\mathrm{avg}}_\theta(s,t)
-\alpha\big(\dot g(t)-J^{-1}\omega_t\big)+O(\alpha^{2})\Big)\Big\|_2^{2}
\nonumber\\
&=\frac{1}{\alpha^{2}}
\Big\|\alpha\big(\dot g(t)-J^{-1}\omega_t\big)+O(\alpha^{2})\Big\|_2^{2}
\nonumber\\
&=\frac{1}{\alpha^{2}}
\Big(\alpha^{2}\big\|\dot g(t)-J^{-1}\omega_t\big\|_2^{2}
+O(\alpha^{3})\Big)
\nonumber\\
&=\big\|\dot g(t)-J^{-1}\omega_t\big\|_2^{2}+O(\alpha),
\nonumber
\end{align}
which is \eqref{eq:af-loss-limit}.
\end{proof}

\paragraph{Numerical verification of Proposition~\ref{prop:af-meanflow-limit}.} Figure~\ref{fig:af-meanflow-limit} verifies Eq.~\eqref{eq:af-loss-limit} on the released implementation. The scaled loss $\widetilde{\mathcal{L}}_{\mathrm{rot}}(\alpha)$ of \eqref{eq:af-rot-loss-theory} is computed with the target evaluated \emph{exactly}, through the group composition of the ratio-$\alpha$ consistency step --- not through the expansion \eqref{eq:af-tgt-expand} being tested. Since the proposition is an identity valid for any $C^{2}$ model, no training is required: we instantiate synthetic smooth two-time velocity heads $(R_t,x_t,t,s)\mapsto\omega^{\mathrm{avg}}_\theta$, random $C^{\infty}$ functions of the state and both times so that the total-derivative cross terms in $\dot g$ are exercised. The gap to the differential MeanFlow loss \eqref{eq:diff-target} vanishes linearly in $\alpha$ (empirical log--log slope $0.98$, matching the $O(\alpha)$ remainder), and at $\alpha=1$ the scaled loss equals the flow-matching loss to machine precision ($\sim\!10^{-16}$), since the target reduces to $\omega_t$ exactly. The inverse Jacobian in the regression target of \eqref{eq:diff-target} is essential to this limit: replacing $J^{-1}\omega_t$ by $\omega_t$ (equivalently, measuring the gap to the residual $J\dot g-\omega_t$) leaves a nonzero plateau ($\approx2\times10^{-2}$ in this setup) as $\alpha\to0$, so the $\alpha$-Flow objective converges to the differential MeanFlow loss \eqref{eq:diff-loss} itself and not to its Jacobian-free variant.

\begin{figure}[t]
  \centering
  \includegraphics[width=0.5\linewidth]{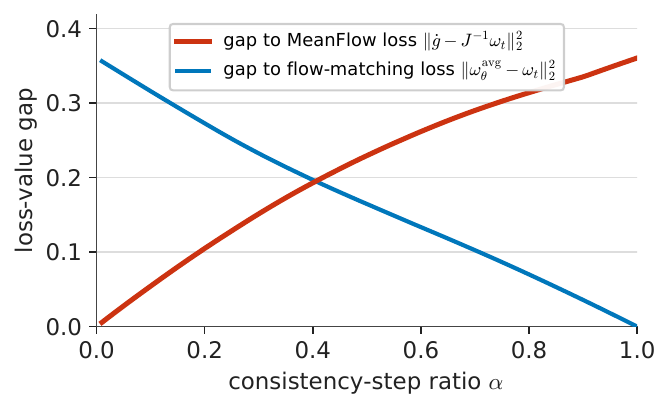}
    \caption{Numerical verification of Proposition~\ref{prop:af-meanflow-limit}, Eq.~\eqref{eq:af-loss-limit}: gap between the scaled $\alpha$-Flow rotation loss $\widetilde{\mathcal{L}}_{\mathrm{rot}}(\alpha)$ and its two predicted limits, averaged over $256$ random $\mathrm{SE}(3)$ instances ($\alpha\in[10^{-2},1]$, linear grid). The gap to the differential MeanFlow loss $\|\dot g-J^{-1}\omega_t\|_2^2$ (blue) vanishes as $O(\alpha)$; the gap to the flow-matching loss (orange) vanishes at $\alpha=1$: the $\alpha$-Flow objective interpolates between the two.}
  \label{fig:af-meanflow-limit}
\end{figure}

\begin{remark}[The fallback branch does not admit this limit]
\label{rmk:af-fallback-limit}
Proposition~\ref{prop:af-meanflow-limit} concerns the group-composition branch of
\eqref{eq:af-rot-target}. The first-order fallback
$\omega^{s\to t}_{\mathrm{tgt}}=\delta\,\tilde\omega_t+\omega^{s\to m}_\theta$
drops the BCH correction, and repeating the argument above with it in place
yields \eqref{eq:af-loss-limit} with $J\equiv I$ --- i.e.,\ the wrong identity. The
fallback is therefore a numerical device for $t-s<t_\varepsilon$, where the two
branches agree to $O(t-s)$ and $\log(\cdot)$ near the identity loses precision,
and not an alternative formulation of the target.
\end{remark}

\begin{remark}[The adjoint action on $\mathfrak{so}(3)$]
\label{rmk:adjoint}
For a Lie algebra element $Y$, the adjoint action $\mathrm{ad}_Y$ is the linear
operator $\mathrm{ad}_Y(X):=[Y,X]=YX-XY$, and an analytic function $f$ applied to
$\mathrm{ad}_Y$ is understood through its power series, each term being an
iterated bracket. On $\mathfrak{so}(3)$ the bracket is the cross product,
$$
\mathrm{ad}_{[g]}\big([\psi]\big)=[g][\psi]-[\psi][g]=[\,g\times\psi\,]=\big[[g]\,\psi\big],
$$
so under the identification $\mathfrak{so}(3)\cong\mathbb{R}^3$ the operator
$\mathrm{ad}_{[g]}$ \emph{is} the matrix $[g]$ acting on $\mathbb{R}^3$. Any
$f(\mathrm{ad}_{[g]})$ therefore reduces to the $3\times3$ matrix $f([g])$ --- a
simplification specific to $\mathfrak{so}(3)$ that we use below. In particular,
comparing with \eqref{eq:left_jacobian},
\begin{align}
J(\phi)=\int_0^1 e^{\alpha[\phi]}d\alpha=\big(e^{[\phi]}-I\big)[\phi]^{-1},
\qquad
J(\phi)^{-1}=\frac{\mathrm{ad}_{[\phi]}}{e^{\mathrm{ad}_{[\phi]}}-I},
\label{eq:jacobian-adjoint}
\end{align}
where the quotients denote the corresponding power series, since $[\phi]$ is
singular. The left Jacobian is thus the operator that transports a Lie-algebra
perturbation to the group, and it appears both as the Fr\'echet derivative of
$\exp$ (Proposition~\ref{pro:Omega_identity}) and as the first-order
Baker--Campbell--Hausdorff coefficient (above); these are two readings of the same object.
\end{remark} 
\begin{remark}[Why $\alpha$ is floored]
\label{rmk:af-floor}
The $1/\alpha$ normalization makes $\mathcal{L}^{\alpha}_{\mathrm{rot}}=O(\alpha)$,
so the loss vanishes for \emph{every} $\theta$ at exactly $\alpha=0$: the limit is
informative only as the leading-order direction, not as a usable objective at
$\alpha=0$. We therefore sample $\alpha\in[\alpha_{\min},1]$ with
$\alpha_{\min}>0$; the two endpoints interpolate between flow matching
($\alpha=1$) and the differential MeanFlow target ($\alpha\to0$).
\end{remark}

\begin{algorithm}[t]
\caption{JVP-free SE(3) $\alpha$-Flow training step}
\label{alg:af-training}
\begin{algorithmic}[1]
\Require state $(R_t,x_t)$; times $s\le t$; data velocities $(\omega_t,v_t)$;
         field $f_\theta$; $\alpha\in[\alpha_{\min},1]$
\State $m\gets\alpha s+(1-\alpha)t$,\quad $\delta\gets t-m$
\State $R_m\gets \exp(-\delta[\omega_t])\,R_t$,\quad
       $x_m\gets x_t-\delta\,v_t$
\State $(\omega^{\mathrm{avg}}_m,v^{\mathrm{avg}}_m)\gets
       \mathrm{sg}\;f_\theta(s,m,R_m,x_m)$
\State $\omega^{s\to m}\gets(m-s)\,\omega^{\mathrm{avg}}_m$,\quad
       $p^{s\to m}\gets(m-s)\,v^{\mathrm{avg}}_m$
\State $\omega^{s\to t}_{\mathrm{tgt}}\gets
       \log\big(\exp(\delta[\omega_t])\,
       \exp([\omega^{s\to m}])\big)^\vee$
\State $p^{s\to t}_{\mathrm{tgt}}\gets \delta\,v_t+p^{s\to m}$
\State $(\omega^{\mathrm{avg}}_\theta,v^{\mathrm{avg}}_\theta)\gets
       f_\theta(s,t,R_t,x_t)$ \Comment{only gradient path}
\State $\omega^{s\to t}_\theta\gets(t-s)\,\omega^{\mathrm{avg}}_\theta$,\quad
       $p^{s\to t}_\theta\gets(t-s)\,v^{\mathrm{avg}}_\theta$
\State Form $\mathcal{L}^{\alpha}_{\mathrm{rot}},
       \mathcal{L}^{\alpha}_{\mathrm{trans}}$
       (Eqs.~\eqref{eq:af-rot-loss},~\eqref{eq:af-trans-loss})
\end{algorithmic}
\end{algorithm}


\section{Semigroup Formulation of SE(3)-MeanFlow and the Semigroup Loss}
\label{sec:semigroup}

The MeanFlow identity \eqref{eq:omega_identity} is the \emph{differential} form of
our average velocity: it is obtained by differentiating \eqref{eq:avg_omega} with respect to
$t$, and its training target carries both a Jacobian--vector product (for
$\tfrac{d}{dt}\omega^{s\to t}_\theta$) and the left Jacobian $J$. In this section
we show that the \emph{same} definition \eqref{eq:avg_omega} also admits a
\emph{finite}, derivative-free characterization: the average-velocity flow map is
a genuine two-parameter semigroup on $\mathrm{SE}(3)$, and consistency with this
semigroup yields a JVP-free, $J$-free training objective. Throughout
$r\le s\le t\in[0,1]$, and $\omega_t=\log(R_1R_0^\top)^\vee$, $v_t=x_1-x_0$ denote
the data-side instantaneous velocities of the geodesic interpolation
\eqref{eq:geodesic_R}--\eqref{eq:geodesic_x}, which are constant along the path.

\subsection{The Average-Velocity Flow Map}

Recall from \eqref{eq:avg_omega} that the average angular velocity over $[s,t]$
is defined by
\begin{align}
\exp\big([\omega^{s\to t}]\big)
=\mathcal{T}\!\exp\Big(\int_s^t [\omega(\tau,R_\tau,x_\tau)]\,d\tau\Big)
=:D(s,t)\in\mathrm{SO}(3),
\label{eq:sg-D}
\end{align}
where the last equality $D(s,t)=R_tR_s^\top$ holds by integrating the spatial ODE
$\dot R_\tau=[\omega_\tau]R_\tau$; this is \eqref{eq:af-Dab}. We write
$\omega^{s\to t}=(t-s)\,\omega^{\mathrm{avg}}(s,t,T_t)\in\mathbb{R}^3$ for the
rotation log-displacement and, on the flat translation branch,
$p^{s\to t}:=x_t-x_s=(t-s)\,v^{\mathrm{avg}}(s,t,T_t)$ for the displacement.

\begin{definition}[Average-velocity flow map]\label{def:flowmap-sg}
For $s\le t$ define $\Psi_{s\to t}:\mathrm{SE}(3)\to\mathrm{SE}(3)$ by
\begin{align}
\Psi_{s\to t}(R,x)
:=\Big(\,D(s,t)\,R,\ \ x+p^{s\to t}\,\Big)
=\Big(\,\exp\big([\omega^{s\to t}]\big)\,R,\ \ x+(t-s)\,v^{\mathrm{avg}}(s,t)\,\Big).
\nonumber
\end{align}
Here $\omega^{s\to t}$ and $p^{s\to t}$ are evaluated at $T_t$, which is determined by $T_s$ along the interpolation path; $\Psi_{s\to t}$ is therefore well defined on that path, where $\Psi_{s\to t}(R_s,x_s)=(R_t,x_t)$ by construction. 
The sampling map $\Phi_{t\to s}$ of \eqref{eq:flow-map} is its inverse,
$\Phi_{t\to s}=\Psi_{s\to t}^{-1}$: the flow map runs from data toward the prior,
while generation runs the other way.
\end{definition}

\subsection{The Semigroup Property}

\begin{proposition}[Semigroup property of the average-velocity flow]
\label{prop:semigroup}
The family $\{\Psi_{s\to t}\}_{0\le s\le t\le1}$ of
Definition~\ref{def:flowmap-sg} satisfies, for all $r\le s\le t$,
\begin{align}
\Psi_{s\to t}\circ\Psi_{r\to s}=\Psi_{r\to t},
\qquad
\Psi_{t\to t}=\mathrm{id}.
\label{eq:sg-flow}
\end{align}
Equivalently, in terms of the average velocities,
\begin{align}
&\textbf{(rotation)}\quad
\exp\big([\omega^{s\to t}]\big)\,\exp\big([\omega^{r\to s}]\big)
=\exp\big([\omega^{r\to t}]\big),
\label{eq:sg-rot-identity}\\[2pt]
&\textbf{(translation)}\quad
(s-r)\,v^{\mathrm{avg}}(r,s)+(t-s)\,v^{\mathrm{avg}}(s,t)
=(t-r)\,v^{\mathrm{avg}}(r,t).
\label{eq:sg-trans-identity}
\end{align}
\end{proposition}

\begin{proof}
On $\mathrm{SO}(3)$ the transition operators compose by inserting
$R_s^\top R_s=I$:
\[
D(r,t)=R_tR_r^\top=(R_tR_s^\top)(R_sR_r^\top)=D(s,t)\,D(r,s),
\]
the near interval appearing to the left of the far one; this is
Proposition~\ref{prop:af-additivity}. (Equivalently, it is the multiplicativity
of the time-ordered exponential
$\mathcal{T}\!\exp(\int_r^t)=\mathcal{T}\!\exp(\int_s^t)\,\mathcal{T}\!\exp(\int_r^s)$,
so the statement holds for an arbitrary velocity field, not only the geodesic
path.) Substituting $D(a,b)=\exp([\omega^{a\to b}])$ gives
\eqref{eq:sg-rot-identity}. On the flat translation factor displacements add,
$x_t-x_r=(x_t-x_s)+(x_s-x_r)$, which is \eqref{eq:sg-trans-identity}. Combining
the two factors,
$\Psi_{s\to t}(\Psi_{r\to s}(R,x))=\big(D(s,t)D(r,s)R,\;x+p^{r\to s}+p^{s\to t}\big)
=\big(D(r,t)R,\;x+p^{r\to t}\big)=\Psi_{r\to t}(R,x)$, i.e.,\ \eqref{eq:sg-flow}.
\end{proof}

\begin{remark}[Where $\mathrm{SO}(3)$ curvature enters]\label{rmk:sg-bch}
The composition \eqref{eq:sg-rot-identity} is \emph{exact} and independent of
curvature: it is the associativity of group multiplication. Curvature appears only
if one tries to collapse the product of exponentials into a \emph{single}
Lie-algebra sum. By the Baker--Campbell--Hausdorff (BCH) formula,
\[
\omega^{r\to t}
=\log\Big(\exp\big([\omega^{s\to t}]\big)\,
\exp\big([\omega^{r\to s}]\big)\Big)^{\vee}
=\omega^{s\to t}+\omega^{r\to s}
+\tfrac12\,\omega^{s\to t}\times\omega^{r\to s}+\cdots,
\]
so the naive additive law $\omega^{r\to t}=\omega^{r\to s}+\omega^{s\to t}$ holds
only when the two axes are collinear; in general it drops the BCH cross terms. The
translation identity \eqref{eq:sg-trans-identity} has no such correction because
$\mathbb{R}^3$ is abelian. This is the single place where the rotation branch
departs from the Euclidean MeanFlow semigroup.
\end{remark}

\subsection{The Semigroup-Consistency Loss}

Proposition~\ref{prop:semigroup} characterizes the correct average-velocity field
\emph{without} any time derivative: the field is consistent iff its one-step
prediction on $[s,t]$ equals the two-step composition through any intermediate
$m\in[s,t]$. We turn this into a regression objective. Let $f_\theta$ output
$(\omega^{\mathrm{avg}}_\theta,v^{\mathrm{avg}}_\theta)(s,t,T_t)$ and write
$\omega^{s\to t}_\theta=(t-s)\,\omega^{\mathrm{avg}}_\theta$,
$p^{s\to t}_\theta=(t-s)\,v^{\mathrm{avg}}_\theta$. Given $s<m<t$, form the
intermediate state by stepping back the near segment $[m,t]$,
\begin{align}
R_m=\exp\big(-[\omega^{m\to t}_\theta]\big)R_t,
\qquad
x_m=x_t-p^{m\to t}_\theta,\nonumber 
\end{align}
and define the composed (two-step) targets by
\eqref{eq:sg-rot-identity}--\eqref{eq:sg-trans-identity}:
\begin{align}
\omega^{s\to t}_{\mathrm{tgt}}
&:=\log\Big(
\underbrace{\exp\big([\omega^{m\to t}_\theta]\big)}_{D(m,t)}\;
\underbrace{\exp\big([\omega^{s\to m}_\theta]\big)}_{D(s,m)}
\Big)^{\!\vee},
\label{eq:sg-rot-target}\\[2pt]
p^{s\to t}_{\mathrm{tgt}}
&:=p^{s\to m}_\theta+p^{m\to t}_\theta ,
\label{eq:sg-trans-target}
\end{align}
where $\omega^{s\to m}_\theta=(m-s)\,\omega^{\mathrm{avg}}_\theta(s,m,T_m)$ and
$p^{s\to m}_\theta=(m-s)\,v^{\mathrm{avg}}_\theta(s,m,T_m)$ are evaluated at the
stepped-back state. The model regresses its \emph{direct} one-step prediction on
$[s,t]$ onto these stop-gradient targets:
\begin{align}
\mathcal{L}^{\mathrm{sg}}_{\mathrm{rot}}
&=\mathbb{E}_{s<m<t}\,\Big\|\omega^{s\to t}_\theta
-\mathrm{sg}\big(\omega^{s\to t}_{\mathrm{tgt}}\big)\Big\|_2^2,
\label{eq:sg-rot-loss}\\
\mathcal{L}^{\mathrm{sg}}_{\mathrm{trans}}
&=\mathbb{E}_{s<m<t}\,\Big\|p^{s\to t}_\theta
-\mathrm{sg}\big(p^{s\to t}_{\mathrm{tgt}}\big)\Big\|_2^2.
\label{eq:sg-trans-loss}
\end{align}
The semigroup constraint alone admits trivial (collapsed) minimizers; it must be
anchored at the diagonal $s\to t$, where \eqref{eq:sg-D} degenerates to the
instantaneous velocity. This boundary term is exactly flow matching:
\begin{align}
\mathcal{L}^{\mathrm{bd}}_{\mathrm{rot}}
=\mathbb{E}_t\big\|\omega^{\mathrm{avg}}_\theta(t,t,T_t)-\omega_t\big\|_2^2,
\qquad
\mathcal{L}^{\mathrm{bd}}_{\mathrm{trans}}
=\mathbb{E}_t\big\|v^{\mathrm{avg}}_\theta(t,t,T_t)-v_t\big\|_2^2 ,
\label{eq:sg-bd-loss}
\end{align}
and the total objective is
\begin{align}
\mathcal{L}^{\mathrm{sg\text{-}MF}}
=\underbrace{\mathcal{L}^{\mathrm{bd}}_{\mathrm{rot}}
+\mathcal{L}^{\mathrm{bd}}_{\mathrm{trans}}}_{\text{flow matching (anchor)}}
+\underbrace{\mathcal{L}^{\mathrm{sg}}_{\mathrm{rot}}
+\mathcal{L}^{\mathrm{sg}}_{\mathrm{trans}}}_{\text{semigroup consistency}} .
\label{eq:sg-total}
\end{align}
Every term uses only forward evaluations of $f_\theta$ together with
$\exp$ and $\log$ on $\mathrm{SO}(3)$ and addition on $\mathbb{R}^3$:
there is no Jacobian--vector product and no explicit left Jacobian $J$.

\begin{remark}[The normalization scalar must stay outside $\log(\exp\cdot\exp)$]
\label{rmk:sg-scalar}
If one prefers to regress the average velocity
$\omega^{\mathrm{avg}}_\theta(s,t)$ itself rather than the log-displacement, the
target is $\omega^{\mathrm{avg}}_{\mathrm{tgt}}=\tfrac{1}{t-s}\omega^{s\to t}_{\mathrm{tgt}}$
with $\omega^{s\to t}_{\mathrm{tgt}}$ from \eqref{eq:sg-rot-target}. The scalar
$\tfrac{1}{t-s}$ must remain \emph{outside} $\log(\exp\cdot\exp)$:
because the two segment generators do not commute, absorbing it into the two
exponentials would rescale each segment angle and corrupt the BCH cross term, no
longer producing $\log D(s,t)$. On the translation branch the same normalization
is harmlessly absorbed into the linear weights. This is the finite-interval
counterpart of the observation made for \eqref{eq:af-rot-target}.
\end{remark}

In the implementation we impose the same constraint on the transported states
rather than the displacements, replacing
$\mathcal{L}^{\mathrm{sg}}_{\mathrm{rot}}+\mathcal{L}^{\mathrm{sg}}_{\mathrm{trans}}$
by
\begin{align}
\mathcal{L}^{\mathrm{sg}}_{\mathrm{geo}}
=\mathbb{E}_{s<m<t}\Big[\big\|\log\big(R_dR_c^\top\big)^{\vee}\big\|_2^2
+\big\|x_d-x_c\big\|_2^2\Big],
\qquad
\begin{cases}
(R_d,x_d)=\Phi_{t\to s}(T_t),\\
(R_c,x_c)=\Phi_{m\to s}\big(\Phi_{t\to m}(T_t)\big),
\end{cases}
\label{eq:sg-geo-app}
\end{align}
with the composed branch under a stop-gradient. Remark~\ref{rmk:sg-variants}
relates the two forms.

\begin{remark}[Two variants of the semigroup residual]
\label{rmk:sg-variants}
Equations \eqref{eq:sg-rot-loss}--\eqref{eq:sg-trans-loss} impose the semigroup
constraint on the \emph{displacements}. The same constraint can be imposed on
the transported \emph{states}: writing
$R_d=\exp(-[\omega^{s\to t}_\theta])R_t$ and
$R_c=\exp(-[\omega^{s\to t}_{\mathrm{tgt}}])R_t$ for the directly and
compositionally transported rotations, the geodesic residual is
$\|\log(R_dR_c^\top)^\vee\|_2^2$, which is \eqref{eq:sg-loss-geo}. By the BCH
formula,
\[
\log\!\big(R_dR_c^\top\big)^{\vee}
=-\omega^{s\to t}_\theta+\omega^{s\to t}_{\mathrm{tgt}}
-\tfrac12\,\omega^{s\to t}_\theta\times\omega^{s\to t}_{\mathrm{tgt}}+\cdots,
\]
so the two residuals agree to first order and differ from second order on by the
cross term. Both vanish exactly at the field characterized in
Proposition~\ref{prop:sg-consistency}, so they share the same global minimizer;
they differ only in how errors are weighted away from it. The translation branch
is identical in both, since $\mathbb{R}^3$ is abelian. We train with the
geodesic form (Algorithm~\ref{alg:training}).
\end{remark}

\subsection{Consistency: Zero Loss Implies Exact Reconstruction}

The boundary term fixes the instantaneous limit and the semigroup term propagates
it to all intervals; together they pin down the unique correct flow map. The
following is the finite (JVP-free) analogue of Proposition~\ref{pro:loss_zero}.

\begin{proposition}[Uniqueness of the semigroup minimizer]
\label{prop:sg-consistency}
Suppose $\omega^{\mathrm{avg}}_\theta$ is continuous and, along the interpolation
path, satisfies the boundary condition
$\omega^{\mathrm{avg}}_\theta(t,t,T_t)=\omega_t$ for all $t$, together with the
rotation semigroup identity
\[
\exp\big([\omega^{s\to t}_\theta]\big)\,
\exp\big([\omega^{r\to s}_\theta]\big)
=\exp\big([\omega^{r\to t}_\theta]\big),
\qquad \forall\,r\le s\le t,
\]
where $\omega^{a\to b}_\theta:=(b-a)\,\omega^{\mathrm{avg}}_\theta(a,b,T_b)$. Then
$\exp\big([\omega^{s\to t}_\theta]\big)=R_tR_s^\top$ for all $s\le t$; in
particular, with $s=0$, $t=1$,
$\hat R_0:=\exp\big(-[\omega^{0\to1}_\theta]\big)R_1=R_0$. The analogous
statement holds for translation with $v^{\mathrm{avg}}_\theta(t,t)=v_t$.
\end{proposition}

\begin{proof}
Fix $t$ and set $G(s):=\exp\big([\omega^{s\to t}_\theta]\big)\in\mathrm{SO}(3)$,
so $G(t)=I$. For $h>0$ the semigroup identity applied to
$s\le s+h\le t$ gives
$G(s)=G(s+h)\,\exp\big([\omega^{s\to s+h}_\theta]\big)$, hence
\[
G(s+h)=G(s)\,\exp\big(-[\omega^{s\to s+h}_\theta]\big).
\]
By the boundary condition and continuity,
$\omega^{s\to s+h}_\theta=h\,\omega^{\mathrm{avg}}_\theta(s,s+h)=h\,\omega_s+o(h)$,
so $\exp\big(-[\omega^{s\to s+h}_\theta]\big)=I-h[\omega_s]+o(h)$ and
therefore
\[
\frac{d}{ds}G(s)=-\,G(s)\,[\omega_s],\qquad G(t)=I.
\]
On the other hand $D(s,t)=R_tR_s^\top$ obeys, using $\dot R_s=[\omega_s]R_s$ and
skew-symmetry,
\[
\frac{d}{ds}D(s,t)=R_t\big(\partial_s R_s^\top\big)
=-\,R_tR_s^\top[\omega_s]=-\,D(s,t)\,[\omega_s],
\qquad D(t,t)=I.
\]
$G$ and $D(\cdot,t)$ solve the same linear ODE with the same terminal condition,
so by uniqueness $G(s)=D(s,t)=R_tR_s^\top$ for all $s\le t$. Taking $s=0$, $t=1$
gives $\exp\big([\omega^{0\to1}_\theta]\big)=R_1R_0^\top$, i.e.,\
$\hat R_0=R_0$. For translation, the additive cocycle $p_\theta(s,t)$ with
diagonal derivative $v^{\mathrm{avg}}_\theta(t,t)=v_t$ integrates to
$p_\theta(s,t)=x_t-x_s$, giving $\hat x_0=x_0$.
\end{proof}

Proposition~\ref{prop:sg-consistency} shows that \eqref{eq:sg-total} is a valid
training objective: its global minimizer is the exact average-velocity field, and
both ingredients are necessary --- the boundary term supplies the instantaneous
data velocity, and the semigroup term is the curvature-exact propagation rule that
extends it to every interval.

\subsection{Relation to the Differential Identity and to $\alpha$-Flow}

\begin{remark}[Relation to the left-Jacobian identity]\label{rmk:sg-equiv}
The semigroup loss \eqref{eq:sg-total} and the differential MeanFlow loss
\eqref{eq:diff-loss} have the same minimizer but realize the left Jacobian
differently. Differentiating the boundary-anchored semigroup identity
\eqref{eq:sg-rot-identity} with respect to $t$ at $s\to t$ reproduces the differential identity
\eqref{eq:omega_identity}, in which $J$ appears explicitly as the Fr\'echet
derivative of $\exp$. In the finite form $J$ never appears as a separate factor:
its entire effect is absorbed into the BCH series of
$\log(\exp\cdot\exp)$ in \eqref{eq:sg-rot-target}. The two
formulations are thus gradient-equivalent to first order, but the finite form is
JVP-free and, because all segment angles are bounded by $\pi$, requires no
differentiation of $f_\theta$ at any point.
\end{remark}

\begin{remark}[$\alpha$-Flow as an instance]\label{rmk:sg-alpha}
The JVP-free $\alpha$-Flow objective of Section~\ref{sec:alpha-flow} is a
particular instantiation of \eqref{eq:sg-total}. Choosing the split point
$m=\alpha s+(1-\alpha)t$ and replacing the near-segment model prediction by the
\emph{data} velocity $\tilde\omega_t=\omega_t$ (so that
$D(m,t)=\exp(\delta[\omega_t])$ with $\delta=\alpha(t-s)$), while keeping
the far segment as the stop-gradient model evaluation, turns
\eqref{eq:sg-rot-target} into the $\alpha$-Flow target \eqref{eq:af-rot-target}
and \eqref{eq:sg-trans-target} into its translation counterpart. The limit
$\alpha=1$ makes the far segment vanish and reduces \eqref{eq:sg-total} to the
flow-matching boundary \eqref{eq:sg-bd-loss}; the limit $\alpha\to0$ recovers, via
BCH, the differential loss (Proposition~\ref{prop:af-meanflow-limit}). The
``pure'' semigroup form \eqref{eq:sg-rot-target}--\eqref{eq:sg-trans-target},
using the model on both segments, corresponds to the Shortcut choice
$\tilde\omega_t=\omega^{\mathrm{avg}}_\theta(m,t,T_t)$.
\end{remark}

\section{Related Work}\label{sec:related}

\paragraph{Earlier and alternative 6-DoF grasp generators.} Learned 6-DoF parallel-jaw grasp synthesis spans several paradigms that predate or complement diffusion and flow-based generators. Candidate-based methods first construct grasp hypotheses and then learn to score them. GPD samples geometrically plausible candidates and classifies them using local point-cloud geometry~\citep{tenpas2017grasp}, while PointNetGPD uses a PointNet-based evaluator on the points observed within the gripper volume~\citep{liang2019pointnetgpd}. 6-DoF GraspNet introduced a distributional alternative based on a conditional variational autoencoder, together with a learned evaluator for scoring and iterative refinement~\citep{mousavian20196}. A complementary line amortizes candidate generation using dense or single-shot predictors: S4G directly predicts scene-level $\mathrm{SE}(3)$ grasp proposals~\citep{qin2020s4g}, Contact-GraspNet anchors full poses at observed contact points~\citep{sundermeyer2021contactgraspnet}, and VGN predicts grasp quality, orientation, and gripper width over a TSDF volume~\citep{breyer2021vgn}. CAPGrasp is a more recent task-specific generative sampler that enforces $\mathbb{R}^3\!\times\!\mathrm{SO}(2)$ equivariance under a prescribed approach-direction constraint~\citep{weng2024capgrasp}. 

These methods are complementary to our approach but are not included in the NFE-matched comparison. Their computational budgets are governed by candidate or anchor counts, evaluator or refinement passes, and method-specific post-processing rather than by evaluations of an iterative pose-space field; CAPGrasp additionally targets approach-constrained generation. We therefore restrict the NFE sweep to iterative generative baselines evaluated under the shared pose-distribution protocol.

\paragraph{EquiGraspFlow.}
\citet{lim2024equigraspflow} generate grasp poses with a conditional continuous normalizing flow on $\mathrm{SE}(3)$, driven by time-dependent angular and linear velocity fields (their time convention runs prior $t{=}0 \to$ data $t{=}1$; we keep their notation in this subsection):
\begin{align}
\dot{R} = [\omega_\theta(t, \mathcal{P}, T)]\,R, \qquad \dot{x} = v_\phi(t, \mathcal{P}, T), \qquad t \in [0,1], \nonumber
\end{align}
which transports the prior $p_0(T|\mathcal{P}) = p_0(R)\,p_0(x|\mathcal{P})$, uniform on $\mathrm{SO}(3)$ times a Gaussian centered at the point-cloud mean, to the distribution of successful grasps. Training uses flow matching with a geodesic rotation path and a straight translation path, whose per-sample targets are
\begin{align}
[\omega^*_t(R \mid R_1)] = \frac{\log(R^\top R_1)}{1-t}, \qquad v^*_t(x \mid x_1) = \frac{x_1 - x}{1-t}, \nonumber
\end{align}
and the loss
\begin{align}
\mathcal{L} = \mathbb{E}_{t,\,T_1 \sim q(T|\mathcal{P}),\,T \sim p_t(T|T_1)} \Big[ \tfrac{1}{2}\big\lVert [\omega_\theta(t,\mathcal{P},T)] - [\omega^*_t(R|R_1)] \big\rVert_F^2 + \big\lVert v_\phi(t,\mathcal{P},T) - v^*_t(x|x_1) \big\rVert^2 \Big]. \nonumber
\end{align}
Equivariance is obtained by construction: if the prior is $\mathrm{SE}(3)$-invariant and the fields satisfy $u(t, T'\mathcal{P}, T'T) = R'\,u(t,\mathcal{P},T)$ for all $T' = (R', x') \in \mathrm{SE}(3)$, then every intermediate distribution $p_t(T|\mathcal{P})$ is $\mathrm{SE}(3)$-invariant (their Prop.~1). This is implemented with a VN-DGCNN encoder, VN-MLP heads, and an equivariant lifting layer that maps the scalar time onto a learned equivariant direction. Conditioning uses classifier-free guidance for flow matching (condition dropout $20\%$; guidance $\beta{=}2$ at sampling), and the native sampler is a fourth-order Runge--Kutta--Munthe-Kaas integrator with 20 steps (i.e.,, 80 field evaluations).

\emph{Relation to our work.} GraspMeanFlow keeps this entire construction --- encoder, heads, prior, guidance, and the equivariance argument --- changes only what the field \emph{means}: instead of the instantaneous velocity at a single time, our field outputs the interval-averaged velocity conditioned on $(t, h{=}t{-}s)$, lifted through the same equivariant-direction mechanism as their time input. Equivariance therefore carries over unchanged, while one evaluation of the field executes a full interval jump rather than one integration step.

\paragraph{SE(3)-DiffusionFields.}
\citet{urain2023se3dif} learn a smooth energy field $E_{\theta}(H, k)$ over grasp poses $H \in \mathrm{SE}(3)$ (conditioned on a learned object shape code) and define the score as its Lie derivative, $s_\theta(H,k) = -DE_\theta(H,k)/DH \in \mathbb{R}^6$. Training is denoising score matching with Gaussian perturbations in the Lie algebra: with noise scales $\sigma_1 < \cdots < \sigma_L$,

\begin{align}
\hat{H} = H\,\mathrm{Expmap}(\epsilon), \quad \epsilon \sim \mathcal{N}(0, \sigma_k^2 I),
\qquad
\mathcal{L}_{\mathrm{dsm}} =
\frac{1}{L}\sum_{k} \mathbb{E}_{H, \hat{H}}
\left\lVert s_\theta(\hat{H}, k) -
\frac{D \log q(\hat{H} \mid H, \sigma_k I)}{D\hat{H}} \right\rVert,
\nonumber
\end{align}
jointly with an object-SDF regression. Sampling runs inverse annealed Langevin MCMC on $\mathrm{SE}(3)$,
\begin{align}
H_{k-1} = \mathrm{Expmap}\!\Big( \tfrac{\alpha_k^2}{2}\, s_\theta(H_k, k)
+ \alpha_k\, \epsilon \Big) H_k,
\qquad \epsilon \sim \mathcal{N}(0, I),
\nonumber
\end{align}
annealing from the largest noise scale down; the learned energy doubles as a differentiable grasp cost for joint grasp-and-motion optimization.

\emph{Relation to our work.} Because it is an annealed MCMC sampler, generation quality depends on a long chain (its native operating point in our protocol is $T{=}240$), and it degrades sharply in the few-step regime (Table~\ref{tab:success}). Our simulated success protocol ports their public \texttt{isaac\_evaluation} harness, applied identically to every method (Appendix~\ref{sec:sim}).

\paragraph{BRIDGER.}
\citet{chen2024bridger} replace the uninformative noise prior of diffusion policies with an arbitrary \emph{source} policy via stochastic interpolants. With source samples $a_0 \sim \pi_0(\cdot|x)$ (for grasping, a hand-crafted heuristic prior that places poses around the
object) and expert samples $a_1 \sim \pi_1(\cdot|x)$, the interpolant
\begin{align}
a_t = \alpha(t)\,a_0 + \beta(t)\,a_1 + \gamma(t)\,z,
\qquad z \sim \mathcal{N}(0, I),
\qquad \gamma(t) = d\sqrt{2t(1-t)},
\nonumber
\end{align}
(with, e.g., $\alpha(t) = (1-t)^3$, $\beta(t) = 1-(1-t)^3$) connects the two at $t{:}\,0 \to 1$. Two networks are trained --- a velocity $b_\theta$ regressing $\partial_t I + \dot{\gamma} z$ and a
reparameterized score $\hat{s}_\eta$ regressing $-z$ --- and generation integrates the forward SDE
\begin{align}
da_t = \big[ b(t, a_t, x) + \epsilon(t)\, s(t, a_t, x) \big]\, dt
+ \sqrt{2\epsilon(t)}\, dW_t
\nonumber
\end{align}
by Euler--Maruyama from a source sample.

\emph{Relation to our work.} BRIDGER injects prior knowledge through the source distribution rather than the objective; its released sampler additionally suppresses noise over the final $30\%$ of the trajectory and amplifies the learned drift, which explains its non-monotone behavior at very small step counts (Appendix~\ref{sec:protocol}). Our approach is complementary: we keep the standard noise-to-data formulation and make the \emph{field itself} represent finite intervals, which is what enables few-step generation without an informative prior.


\paragraph{Riemannian MeanFlow (concurrent work).}
Two concurrent works, both titled Riemannian MeanFlow \citep{zhong2026riemannian,woo2026riemannian}, extend MeanFlow to general Riemannian manifolds. Both are manifold-generative methods evaluated on synthetic geometries (spheres, tori, $\mathrm{SO}(3)$), DNA sequence design, and unconditional protein backbones; neither targets 6-DoF grasp synthesis, point-cloud conditioning, or $\mathrm{SE}(3)$-equivariance. They also differ from our formulation in how the interval-averaged velocity is defined. \citet{zhong2026riemannian} define it by parallel-transporting instantaneous velocities along the trajectory into a common tangent space, which is \emph{path-dependent}: it requires the full trajectory $\{x_\tau\}$ rather than the endpoints alone. We instead define the average velocity through the time-ordered exponential $\exp\big((t\!-\!s)[\omega^{\mathrm{avg}}]\big) =\mathcal{T}\!\exp\!\big(\int_s^t[\omega_\tau]\,\mathrm{d}\tau\big)$, which depends only on the endpoints $R_s,R_t$ and closes exactly under a single exponential step; the two definitions agree only for a constant (geodesic) field, and otherwise differ because the transported integral accumulates the tangent contributions along the path while ours depends only on the endpoints. For training, \citet{woo2026riemannian} use an endpoint semigroup objective on a sampled intermediate time. Our Stage-2 objective is closest to this, but is anchored differently: we pair semigroup consistency with an explicit flow-matching boundary term \eqref{eq:bd-loss} on the diagonal $s=t$, which ties the field to the marginal flow at every $t$, and we warm up with an $\alpha$-Flow stage \eqref{eq:af-target-main} whose near segment is supplied by the data velocity.

\paragraph{SE(3)-MeanFlow for protein backbones (concurrent work).}
\citet{bai2026se3meanflow} use the same $\mathrm{SE}(3)$ average-velocity formulation as ours, including the time-ordered exponential definition and the $\alpha$-Flow warm-up, for few-step protein backbone generation. The difference is the frame: they express the Lie-algebra velocity in the body frame $\dot R_t=R_t[\omega_t]$, whereas we use the spatial frame $\dot R_t=[\omega_t]R_t$, since under a scene rotation the spatial velocity transforms covariantly and matches the output type of a vector-neuron head, while the body-frame velocity is rotation-invariant (Appendix~\ref{sec:arch}).

\section{Architecture}
\label{sec:arch}
Our network is the EquiGraspFlow~\citep{lim2024equigraspflow} backbone, unchanged except for the additions required to turn a single-time flow-matching model into a two-time MeanFlow model. This section describes the modifications and explains why the $\mathrm{SE}(3)$-equivariance of the original design is preserved; the formal statement and proof are given in Proposition~\ref{prop:net-equivariant}.

\paragraph{Backbone.}
The point-cloud encoder is a VN-DGCNN with channel widths $1{-}21{-}21{-}42{-}85{-}170{-}341$ and $k{=}40$ neighbors, producing an equivariant object feature $z \in \mathbb{R}^{341\times 3}$; the velocity head is a vector-neuron MLP with widths $347{-}256{-}256{-}128{-}128{-}128{-}2$ that maps the object feature, the current pose (fed as the vector channels $(R_1, R_2, R_3, x)$, the columns of $R$ plus the translation), and the time inputs to a rotational and a translational velocity. Translation equivariance is handled, as in EquiGraspFlow, by centering the point cloud and pose at the point-cloud mean. The trunk, widths, and all other components are inherited unchanged; the full model has $633{,}803$ parameters versus $633{,}291$ for the backbone --- the modification below accounts for 512 parameters (${<}0.1\%$).

\paragraph{Two-time conditioning.}
A MeanFlow field is conditioned on a time pair rather than a single time. We condition on $(t, h)$ with $h = t - s$ (current time and interval length, following MeanFlow), reusing EquiGraspFlow's equivariant lifting mechanism: each scalar multiplies a single shared learned equivariant direction $d = f_{\mathrm{equi}}(V)$, a linear vector-neuron projection of the concatenated feature and pose channels $V$. The two resulting channels $t\,d$ and $h\,d$ are appended to the head input (width 347 vs.\ the baseline's 346, which appends only $t\,d$). Because $f_{\mathrm{equi}}(V Q^\top) = f_{\mathrm{equi}}(V)\,Q^\top$ for any rotation $Q$ and the times are rotation-invariant scalars, the added channels transform exactly like every other vector channel, so the field remains equivariant with two time inputs for the same reason it was with one.

\paragraph{Output frame and equivariance.}
Throughout, angular velocities are expressed in the \emph{spatial} frame
($\dot R=[\omega]\,R$, the standard robotics convention), and our formulation,
training targets, and samplers are all built in this frame. The head therefore
outputs the rotational velocity in the spatial frame, and the sampler applies it
directly through the left-invariant exponential update
$\exp\!\big(-(t-s)[\omega^{\mathrm{avg}}]\big)R$, with no change of frame.
This choice is forced by equivariance: under a scene rotation $Q$, the spatial
angular velocity transforms covariantly ($\omega\mapsto Q\omega$), matching the
transformation type of a vector-neuron output. The body-frame velocity
$\omega_{\mathrm{body}}=\mathrm{vee}(R^\top\dot R)$, by contrast, is
rotation-\emph{invariant}, since $R^\top\dot R$ is unchanged when $R\mapsto QR$.
An equivariant vector head regressing a rotation-invariant target has zero
conditional mean over rotation-augmented data, so a body-frame parameterization
cannot represent the average-velocity field with an equivariant head; the spatial
parameterization is what makes an equivariant average-velocity field trainable.

\paragraph{Distributional equivariance.}
The prior is $\mathrm{SE}(3)$-invariant (uniform rotations; translations from a standard Gaussian, which is centered at the point-cloud centroid because the loader zero-centers every cloud), the guided field is a convex combination of two equivariant fields (the null condition $z{=}0$ is itself fixed under rotations), and every inference update --- the Euler jump and the exponential-$\mathrm{SO}(3)$ endpoint step --- composes equivariant maps. By the same argument as EquiGraspFlow's Prop.~1, extended to interval-averaged velocities in Proposition~\ref{prop:net-equivariant}, transporting an invariant prior through an equivariant flow map yields an $\mathrm{SE}(3)$-invariant conditional grasp distribution: rotating the object rotates the generated grasp set identically.

\section{Theoretical Properties of Equivariant Flows}

\label{sec:mf-theory}

\begin{proposition}[Equivariance of the two-time velocity field]
\label{prop:net-equivariant}
The two-time average-velocity field
$(\omega^{\mathrm{avg}}_\theta,v^{\mathrm{avg}}_\theta)(s,t,\mathcal{P},T)$
of Appendix~\ref{sec:arch} is $\mathrm{SE}(3)$-equivariant in the sense of
Definition~\ref{def:equi-field}: for every $T'=(R',x')\in\mathrm{SE}(3)$,
\[
(\omega^{\mathrm{avg}}_\theta,v^{\mathrm{avg}}_\theta)(s,t,T'\mathcal{P},T'T)
=R'\,(\omega^{\mathrm{avg}}_\theta,v^{\mathrm{avg}}_\theta)(s,t,\mathcal{P},T).
\]
\end{proposition}

\begin{figure}
    \centering
    \includegraphics[width=0.5\linewidth]{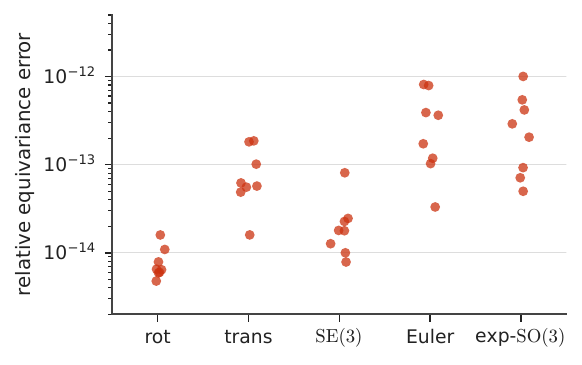}
    \caption{Equivariance of the trained model in double precision. Each dot is one random $T'$: the first three columns probe the two-time field under rotation-only, translation-only, and full $\mathrm{SE}(3)$ actions; the last two probe the composed guided samplers ($\mathrm{NFE}\,5$). All deviations sit at machine precision.} \label{fig:equiv-check}
\end{figure}

\paragraph{Numerical validation of Proposition~\ref{prop:net-equivariant}.} We apply random rigid transformations (Haar-uniform rotations, translations $\sim\mathcal{N}(0,3^{2}I)$) to the point cloud and grasp poses of a test instance, re-apply the pipeline's mean-centering, and compare the trained network's output against $R'$ times the untransformed output at $16$ poses and random time pairs. The field deviates by at most $8\times10^{-14}$, and running the Euler and exp-$\mathrm{SO}(3)$ samplers from the transformed inputs and the correspondingly transformed prior noise reproduces the $T'$-transformed grasps to $1\times10^{-12}$ (Figure~\ref{fig:equiv-check}) --- the flow-map equivariance of Proposition~\ref{prop:flowmap-equi} realized by the implementation. In single precision, the translation test degrades to $8\times10^{-4}$ because of $k$-nearest-neighbor ties in the encoder graph under the centering round-off, rather than because of the architecture.

\begin{lemma}[The single-time VN backbone is $\mathrm{SE}(3)$-equivariant]
\label{lem:vn-equivariant}
Let a vector-list feature be $V=(v_1,\dots,v_C)\in(\mathbb{R}^3)^C$, and let
$Q\in\mathrm{SO}(3)$ act channelwise, $Q\cdot V:=(Qv_1,\dots,Qv_C)$. Every layer
of the vector-neuron backbone of \citet{lim2024equigraspflow}---the equivariant
encoder $\Phi$, the channel-linear map $\mathrm{VNLin}$, and the nonlinearity
$\mathrm{VNReLU}$---commutes with this action,
\[
\Phi(Q\cdot\mathcal{P})=Q\cdot\Phi(\mathcal{P}),\qquad
L(Q\cdot V)=Q\cdot L(V)\quad\text{for }L\in\{\mathrm{VNLin},\mathrm{VNReLU}\},
\]
and hence so does any composition of them.
\end{lemma}

\begin{proof}
A channel-linear map acts only on the channel index, $(\mathrm{VNLin}\,V)_j=\sum_c
W_{jc}v_c$, so $\mathrm{VNLin}(Q\cdot V)_j=\sum_c W_{jc}(Qv_c)=Q\sum_c W_{jc}v_c
=(Q\cdot\mathrm{VNLin}\,V)_j$. For $\mathrm{VNReLU}$, write $p=\mathrm{VNLin}_p V$,
$d=\mathrm{VNLin}_d V$; the gate depends on the data only through the inner
products $\langle p_c,d_c\rangle$ and $\langle d_c,d_c\rangle$, which are
invariant since $\langle Qp_c,Qd_c\rangle=\langle p_c,d_c\rangle$. Thus the output
is a combination of the equivariant $p,d$ with invariant scalar coefficients,
whence $\mathrm{VNReLU}(Q\cdot V)=Q\cdot\mathrm{VNReLU}(V)$. Equivariance of $\Phi$
is \citet[Prop.~1]{lim2024equigraspflow}. Composition preserves the intertwining
relation, so the whole backbone is equivariant.
\end{proof}

\begin{lemma}[The two-time injection preserves equivariance]
\label{lem:two-time-injection}
Fix $(s,t)$ and set $h:=t-s$. Let $L=\mathrm{VNLin}$ be equivariant and define the
injection
\[
\iota(V,s,t):=\bigl[\,V\ ;\ t\,L(V)\ ;\ h\,L(V)\,\bigr],
\]
appending the two channels $t\,L(V)$ and $h\,L(V)$ to the vector-list $V$. Then
$\iota$ is $\mathrm{SO}(3)$-equivariant in $V$:
$\iota(Q\cdot V,s,t)=Q\cdot\iota(V,s,t)$ for all $Q\in\mathrm{SO}(3)$.
\end{lemma}

\begin{proof}
The scalars $t$ and $h=t-s$ are functions of $(s,t)$ alone, hence invariant under
the action on $V$. By Lemma~\ref{lem:vn-equivariant}, $L(Q\cdot V)=Q\,L(V)$, so
each appended channel transforms covariantly,
\[
t\,L(Q\cdot V)=t\,\bigl(Q\,L(V)\bigr)=Q\,\bigl(t\,L(V)\bigr),\qquad
h\,L(Q\cdot V)=Q\,\bigl(h\,L(V)\bigr),
\]
using that $t,h$ are invariant scalars commuting with $Q$. Since concatenation of
covariant channels is covariant, $\iota(Q\cdot V,s,t)=Q\cdot\iota(V,s,t)$.
\end{proof}

\begin{proof}[Proof of Proposition~\ref{prop:net-equivariant}]
The point cloud and pose are mean-centered before the network, which removes the
absolute translation $x'$; it therefore suffices to treat $T'=(R',0)$, so that the
scene action is $\mathcal{P}\mapsto R'\mathcal{P}$, $T=(R,x)\mapsto R'T=(R'R,R'x)$,
while $(s,t)$ is unchanged.

Write the network as the composition
\[
(\omega^{\mathrm{avg}}_\theta,v^{\mathrm{avg}}_\theta)(s,t,\mathcal{P},T)
=F\Bigl(\iota\bigl(V(\mathcal{P},T),\,s,t\bigr)\Bigr),\qquad
V(\mathcal{P},T)=\bigl[\Phi(\mathcal{P})\ ;\ (R,x)\bigr],
\]
where $V$ collects the equivariant encoder features and the pose vectors,
$\iota$ is the two-time injection of Lemma~\ref{lem:two-time-injection}, and $F$
is the VN-MLP head whose six output coordinates are read as the two equivariant
$3$-vectors $(\omega^{\mathrm{avg}}_\theta,v^{\mathrm{avg}}_\theta)$.

Under $T'=(R',0)$ the input vector-list transforms covariantly,
$V(R'\mathcal{P},R'T)=R'\cdot V(\mathcal{P},T)$, by equivariance of $\Phi$
(Lemma~\ref{lem:vn-equivariant}) and by $R\mapsto R'R$, $x\mapsto R'x$. Applying
Lemma~\ref{lem:two-time-injection} and then Lemma~\ref{lem:vn-equivariant} to
$F$,
\[
F\bigl(\iota(V(R'\mathcal{P},R'T),s,t)\bigr)
=F\bigl(\iota(R'\cdot V(\mathcal{P},T),s,t)\bigr)
=F\bigl(R'\cdot\iota(V(\mathcal{P},T),s,t)\bigr)
=R'\,F\bigl(\iota(V(\mathcal{P},T),s,t)\bigr).
\]
Reading off the two output $3$-vectors, this is exactly
\[
(\omega^{\mathrm{avg}}_\theta,v^{\mathrm{avg}}_\theta)(s,t,T'\mathcal{P},T'T)
=R'\,(\omega^{\mathrm{avg}}_\theta,v^{\mathrm{avg}}_\theta)(s,t,\mathcal{P},T),
\]
which is the claimed equivariance of Definition~\ref{def:equi-field}. The only
departure from the single-time architecture of \citet{lim2024equigraspflow} is
the extra input $h=t-s$, an invariant scalar, which by
Lemma~\ref{lem:two-time-injection} leaves the intertwining relation intact.
\end{proof}

\subsection{Proof of Proposition~\ref{prop:flowmap-equi}}
\begin{proof}
Write $\omega=\omega^{\mathrm{avg}}_\theta(s,t,\mathcal{P},T_t)$, $v=v^{\mathrm{avg}}_\theta(s,t,\mathcal{P},T_t)$, and $h=t-s$. Using
$[R'\omega]=R'[\omega]R'^{\top}$ and hence
$\exp(-h[R'\omega])=R'\,\exp(-h[\omega])\,R'^{\top}$, the rotation
component is $R'\exp(-h[\omega])R'^{\top}R'R_t=R'\exp(-h[\omega])R_t$,
and the translation component is $(R'x_t+x')-hR'v=R'(x_t-hv)+x'$. Closure under
composition is immediate.
\end{proof}

The same argument covers the endpoint-style samplers of
Section~\ref{sec:sampling}, in both their linear and exponential rotation
schedules. Each forms the data-endpoint prediction $\hat R_0$, equivariant by
$\hat R_0(s,t,T'\mathcal{P},T'T_t)=R'\hat R_0(s,t,\mathcal{P},T_t)$, and updates the
rotation as $\exp(-\Delta t\,[A])R_t$ with $A$ proportional to the spatial
log-displacement $\log(R_t\hat R_0^\top)^\vee$. Under a scene rotation this is
covariant,
$\log\big(R'(R_t\hat R_0^\top)R'^\top\big)^\vee=R'\log(R_t\hat R_0^\top)^\vee$,
by the conjugation identity, so $A\mapsto R'A$ and
$\exp(-\Delta t\,[R'A])R'R_t=R'\,\exp(-\Delta t\,[A])R_t$. The two
schedules differ only in the scalar rate on $\log(R_t\hat R_0^\top)^\vee$, which
does not affect equivariance; the translation update is linear and equivariant in
both.

\subsection{Proof of Proposition~\ref{prop:fewstep-inv}}
\begin{proof}
For measurable $A\subseteq\mathrm{SE}(3)$, equivariance gives
$\Phi_{T'\mathcal{P}}^{-1}(T'A)=T'\Phi_{\mathcal{P}}^{-1}(A)$: writing $T=T'S$,
$\Phi_{T'\mathcal{P}}(T'S)=T'\Phi_{\mathcal{P}}(S)\in T'A$ iff
$S\in\Phi_{\mathcal{P}}^{-1}(A)$. Hence
\begin{align*}
(\Phi_{\#}p_1)(T'A\mid T'\mathcal{P})
&=p_1\big(T'\Phi_{\mathcal{P}}^{-1}(A)\mid T'\mathcal{P}\big)\\
&=p_1\big(\Phi_{\mathcal{P}}^{-1}(A)\mid\mathcal{P}\big)
=(\Phi_{\#}p_1)(A\mid\mathcal{P}),
\end{align*}
the middle step by invariance of the prior.
\end{proof}

\section{Optimal-Transport Coupling: Minibatch versus Per-Object}
\label{sec:ot-coupling}

Flow matching is trained on pairs $(T_0,T_1)$ of a data grasp and a prior
sample. Rather than pairing them independently at random, we couple them within
each minibatch by optimal transport~\citep{villani2009optimal}, which
straightens the interpolation paths and lowers gradient
variance~\citep{bose2023se,tong2023improving,pooladian2023multisample}. On $\mathrm{SE}(3)$ we use
the squared geodesic cost
\begin{equation}
    c(T_0,T_1)=w_R\,\big\|\log(R_1R_0^\top)^\vee\big\|_2^2
              +w_x\,\big\|x_0-x_1\big\|_2^2 ,
    \label{eq:ot-cost}
\end{equation}
and solve a linear assignment to obtain the coupling. There are two ways to pose
this assignment, and the choice matters for both correctness and cost.

\paragraph{Global (minibatch) OT.} The standard
choice~\citep{pooladian2023multisample,bose2023se} solves a single assignment
over \emph{all} $B'=\sum_{i} J_i$ grasp--noise pairs in the minibatch, mixing
the $J_i$ grasps of every object $\mathcal{P}_i$ into one problem. In the
conditional setting this is a mismatch on two counts. First, the field is conditioned per object: what it learns is the
transport $p_1(\cdot\mid\mathcal{P}_i)\to q(\cdot\mid\mathcal{P}_i)$ for each
object separately, so the coupling should be a coupling of \emph{those} two
conditionals. A global plan is instead a joint assignment across objects, and
the prior samples it allots to object $i$ are a subset selected for proximity to
that object's grasps, so the per-object prior marginal it induces is no longer
$p_1$. Second, a linear assignment over $B'$ elements costs $O(B'^{3})$, and
$B'$ is the number of objects per batch times the grasps per object: our batch
of $4$ objects $\times$ up to $256$ grasps already gives a $1024\times1024$
problem, and it grows rapidly as the batch mixes more objects.

\paragraph{Per-object OT.} We instead solve one $J_i\times J_i$ assignment
\emph{within} each object, pairing that object's grasps only with that object's
prior samples:
\begin{equation}
    \pi_i=\arg\min_{\pi\in\mathcal{S}_{J_i}}
          \sum_{j=1}^{J_i} c\big(T_{0,ij},\,T_{1,i\pi(j)}\big),
    \qquad i=1,\dots,N_{\mathrm{batch}} .
    \label{eq:per-object-ot}
\end{equation}
Since only a permutation within each object is applied, the prior sample set of
every object is left unchanged and its marginal is exactly $p_1$, matching the
independent-per-object prior used at inference. It is also far cheaper: for a
batch of $n$ objects with $J$ grasps each, global OT costs $O(n^{3}J^{3})$ while
per-object OT costs $O(n\,J^{3})$, an $O(n^{2})$ reduction. We therefore use
per-object OT throughout; with $w_R=w_x=1$, the coupling adds negligible
overhead to each training step.

\paragraph{The coupling preserves equivariance.}
The cost \eqref{eq:ot-cost} is invariant under a global rigid motion: for any
$T'=(R',x')$ one has
$\log\big((R'R_1)(R'R_0)^{\top}\big)=R'\log(R_1R_0^{\top})R'^{\top}$, whose
$\vee$-norm is unchanged, and
$\|(R'x_0{+}x')-(R'x_1{+}x')\|=\|x_0-x_1\|$, so $c(T'T_0,T'T_1)=c(T_0,T_1)$.
Since the prior of Section~\ref{sec:problem} is $\mathrm{SE}(3)$-invariant,
$(R_1,x_1)\sim p_1(\cdot\mid\mathcal{P})$ implies
$T'(R_1,x_1)\sim p_1(\cdot\mid T'\mathcal{P})$, so transforming an object and
its grasps leaves the $J_i\times J_i$ cost matrix entrywise unchanged; the
assignment solver is deterministic given the matrix, so $\pi_i$ is unchanged
even in the presence of ties. The coupled pairs, and therefore the training
targets built from them, are the $T'$-transforms of the untransformed ones: the
coupling does not break the guarantee of Section~\ref{sec:equivariance}.

\section{Training Configuration}
\label{sec:config}
Training runs on a single NVIDIA A100 (40\,GB): 120k steps take ${\sim}7$ hours at ${\sim}204$\,ms/step with a $6.3$\,GB memory footprint. Evaluation and simulation run on nodes with $4{\times}$A100 (one evaluation lane per GPU) using Isaac Gym Preview~4; the software stack is Python~3.10.20 with PyTorch~2.11.0 (CUDA~12.8), Open3D~0.16.0, and POT~0.9.6 under Linux (kernel~6.4). Training uses seed~0; evaluation seeds are $100{+}\mathrm{NFE}$, shared across all methods. Each reported configuration is a single training run, and the objective ablation of Appendix~\ref{sec:ablation} averages three seeds; the lift test is deterministic given the exported grasps. Table~\ref{tab:hyper} lists the hyperparameters of both stages for the two reported models.

\begin{table}[h]
\centering
\small
\setlength{\tabcolsep}{3pt}
\begin{tabular}{lccc}
\toprule
Hyperparameter & Stage 1 ($\alpha$-Flow) & Stage 2: GMF-SG & Stage 2: GMF-JVP \\
\midrule
Objective & $\alpha$-Flow (JVP-free) & \multicolumn{2}{c}{FM boundary + consistency} \\
Initialization & from scratch & \multicolumn{2}{c}{Stage-1} \\
Steps & 18k & 102k & 102k \\
$\alpha$ schedule & $1\to 0.2$ (sigmoid) & --- & --- \\
Boundary term & --- & \multicolumn{2}{c}{flow matching} \\
Consistency term & --- & geodesic semigroup \eqref{eq:sg-loss-geo} & differential identity (JVP) \\
Consistency weight & --- & $1.0$ & $1.7$ \\
Jacobian form & --- & --- & inverse (left) \\
Rotation Huber radius & --- & --- & $100$ \\
Coupling & \multicolumn{3}{c}{independent (GMF-SG) / per-object optimal transport (GMF-JVP), both stages} \\
Loss weights (rot, trans) & \multicolumn{3}{c}{$(0.5, 0.5)$} \\
Loss normalization & \multicolumn{3}{c}{off} \\
Rotation output frame & \multicolumn{3}{c}{spatial} \\
Two-time input ($s\leq t$) & \multicolumn{3}{c}{yes} \\
Time sampler & \multicolumn{3}{c}{$t \sim \mathrm{Unif}(0,1]$, $s \sim \mathrm{Unif}(0,t)$ ($s{=}t$ w.p.\ $0.25$)} \\
CFG dropout / guidance & \multicolumn{3}{c}{$0.2$ / $2.0$} \\
EMA decay & \multicolumn{3}{c}{$0.999$ (EMA weights evaluated)} \\
Optimizer & \multicolumn{3}{c}{Adam, lr $10^{-4}$, weight decay $10^{-6}$} \\
Batch & \multicolumn{3}{c}{4 objects $\times$ up to 256 grasps} \\
Devices & \multicolumn{3}{c}{$1\times$ NVIDIA A100 (40\,GB)} \\
\bottomrule
\end{tabular}
\caption{Training hyperparameters of the two reported models. The
network, schedule, optimizer, and budget are shared; the two differ in the
consistency term, its weight, and the coupling.}
\label{tab:hyper}
\end{table}

\paragraph{Values tried and how the final setting was chosen.}
Hyperparameters carried over from EquiGraspFlow (network widths, optimizer, batch composition, point-cloud scaling) were not re-tuned. The values we did search were tuned under a reduced protocol: a single category trained for 30k steps; the selected setting was then trained at full scale, so no selection was made on the reported four-category runs. For the consistency weight of GMF-JVP, we tried $\{0.05, 1.7\}$ and kept $1.7$, the value at which the two loss terms contribute comparably (their raw magnitudes stand in a ratio of about $0.6$, so the smaller weight leaves the consistency term inactive). For the Huber radius on the rotation residual, we tried $\{20, 50, 100\}$; the three differ by at most $0.007$ EMD, and we kept the loosest, which almost never fires. For the number of objects per step, we tried $\{4, 16\}$ and kept $4$, since the larger batch gave no improvement at a matched number of steps. The $\alpha$-Flow warm-up fraction was tried at $\{0, 0.15\}$ and the training length at $\{120\mathrm{k}, 200\mathrm{k}\}$: the warm-up is optional (Appendix~\ref{sec:ablation}) and 200k steps leave EMD flat and move NFE-20 success by at most $1.6$ points, so we report 18k${+}$102k steps. Selection used the four-category averages of the evaluation protocol: EMD for GMF-SG, and simulated success at NFE~5 for GMF-JVP.

\section{Training Algorithm}
\label{sec:algo}
Algorithm~\ref{alg:training} summarizes one training iteration of the decomposed
SE(3) MeanFlow objective (Stage~2; Stage~1 replaces lines 6--11 with the
$\alpha$-Flow loss of Section~\ref{sec:alpha-flow}, which additionally samples a
second time $s \sim \mathrm{Unif}(0,t)$, set to $t$ with probability $0.25$).
Point clouds are mean-centered as a preprocessing step, so the translation prior of
Section~\ref{sec:problem} is standard normal in the centered space. The map
$\Phi_\theta(s,t,R,x)$ is the average-velocity step \eqref{eq:flow-map}, i.e.,\
$\big(\exp(-(t-s)[\omega^{\mathrm{avg}}_\theta])R,\;
x-(t-s)v^{\mathrm{avg}}_\theta\big)$.

\begin{algorithm}[h]
\caption{GraspMeanFlow training (Stage 2) }
\label{alg:training}
\begin{algorithmic}[1]
\Require dataset $\mathcal{D}$, encoder $f_{\mathrm{enc}}$, two-time
  field $f_\theta$, weight $\lambda_{\mathrm{sg}}$, dropout
  $p_{\mathrm{uncond}}$
\While{not converged}
  \State Sample $(\mathcal{P}, R_0, x_0) \sim \mathcal{D}$;\;\;
    $z \gets f_{\mathrm{enc}}(\mathcal{P})$;\;\;
    $z \gets 0$ w.p.\ $p_{\mathrm{uncond}}$
  \State Sample $R_1 \sim \mathrm{Unif}(\mathrm{SO}(3))$,\;
    $x_1 \sim \mathcal{N}(0, I)$,\;
    $t \sim \mathrm{Unif}(0, 1]$
  \State $\omega_t \gets \log(R_1 R_0^\top)^{\vee}$;\;\;
    $v_t \gets x_1 - x_0$
  \State $R_t \gets \exp\!\big(t\,[\omega_t]\big)\,R_0$;
    \;\; $x_t \gets (1{-}t)\,x_0 + t\,x_1$
  \State $(\omega^{\mathrm{avg}}_\theta, v^{\mathrm{avg}}_\theta)
    \gets f_\theta(t,t,R_t, x_t)$
  \State $\mathcal{L}_{\mathrm{fm}} \gets
    \lVert \omega^{\mathrm{avg}}_\theta - \omega_t \rVert^2
    + \lVert v^{\mathrm{avg}}_\theta - v_t \rVert^2$
  \State $s \sim \mathrm{Unif}(0,t)$;\;\; $m \sim \mathrm{Unif}(s,t)$
  \State $(R_c, x_c) \gets
    \Phi_\theta\big(s, m,\Phi_\theta(m,t,R_t, x_t )\big)$
    \Comment{composed; stop-grad}
  \State $(R_d, x_d) \gets \Phi_\theta(s, t,R_t, x_t)$
    \Comment{direct}
  \State $\mathcal{L}_{\mathrm{sg}} \gets
    \lVert \log(R_d R_c^\top)^{\vee} \rVert^2
    + \lVert x_d - x_c \rVert^2$
  \State Adam step on
    $\mathcal{L}_{\mathrm{fm}} + \lambda_{\mathrm{sg}}
    \mathcal{L}_{\mathrm{sg}}$;\; EMA update
\EndWhile
\end{algorithmic}
\end{algorithm}

\section{Sampling Algorithm}
\label{sec:inference}

We report two samplers. The \emph{Euler} sampler applies \eqref{eq:few-step}
directly on the step schedule: each step subtracts the interval-averaged
velocity over $[t_{k-1},t_k]$, and no pseudocode beyond \eqref{eq:few-step} is
needed. The \emph{endpoint-style} sampler of
Algorithm~\ref{alg:exp-inference} instead queries the field over $[0,t]$ at
every step, forms the data-endpoint prediction that the interval-averaged field
implies, and moves the current pose a fraction $\eta$ of the way toward it. This
is available only because the field is interval-averaged: an instantaneous field
has no endpoint to predict.

Two rotation schedules are covered. The \textsc{linear} schedule uses the
natural remaining-time rate $\eta=\Delta t/t$, which reproduces the Euler update
on the geodesic. The \textsc{exp} schedule uses a constant rate
$\eta=\min(c\,\Delta t,1)$, following the exponential rotation scheduler of
ReQFlow~\citep{yue2025reqflow}; we use $c=10$ and $t_{\min}=10^{-6}$. At the step sizes
reported in the paper ($\mathrm{NFE}\le10$), $\eta$ saturates at $1$, so each
update snaps the rotation to the currently predicted endpoint, which favors
high-feasibility modes when the step budget is small. The accelerated schedule
applies to the rotation only; the translation update is linear in both.
Equivariance of both schedules is verified in Appendix~\ref{sec:mf-theory}.

\begin{algorithm}[h]
\caption{Few-step inference with endpoint prediction (linear and
exponential rotation schedules)}
\label{alg:exp-inference}
\begin{algorithmic}[1]
\Require prior sample $(R_1,x_1)$; guided two-time field
$(\omega^{\mathrm{avg}}_\theta,v^{\mathrm{avg}}_\theta)$; steps $T\ge2$;
$t_{\min}>0$; rotation schedule $\in\{\textsc{linear},\textsc{exp}\}$;
exp rate $c>0$
\Ensure $(\hat R_0,\hat x_0)$
\State $\{t_i\}_{i=0}^{T-1}\gets\mathrm{linspace}(1,t_{\min},T)$
\State $(R,x)\gets(R_1,x_1)$
\For{$i=0,\dots,T-2$}
  \State $t\gets t_i$;\;\; $\Delta t\gets t_i-t_{i+1}$
  \State $(\bar\omega,\bar v)\gets
    (\omega^{\mathrm{avg}}_\theta,v^{\mathrm{avg}}_\theta)(t-\Delta t,t,R,x)$
    \Comment{average velocity over $[t-\Delta t,t]$}
  \State $\hat R_0^\theta\gets\exp(-t[\bar\omega])\,R$;\;\;
    $\hat x_0^\theta\gets x-t\,\bar v$
    \Comment{predicted data endpoint}
  \State $u\gets\log\big(\hat R_0^\theta R^\top\big)^{\vee}$
    \Comment{log-map from $R$ toward $\hat R_0^\theta$}
  \If{schedule $=$ \textsc{linear}}
    \State $\eta\gets\Delta t/t$ \Comment{natural ODE rate (remaining time)}
  \Else
    \State $\eta\gets\min(c\,\Delta t,1)$
      \Comment{\textsc{exp}: constant rate $c$, capped at the endpoint}
  \EndIf
  \State $R\gets\exp(\eta[u])\,R$
  \State $x\gets x-\Delta t\,\bar v$
    \Comment{translation: linear in both schedules}
\EndFor
\State $(\bar\omega,\bar v)\gets
  (\omega^{\mathrm{avg}}_\theta,v^{\mathrm{avg}}_\theta)(0,t_{\min},R,x)$
\State $(\hat R_0,\hat x_0)\gets
  \big(\exp(-t_{\min}[\bar\omega])\,R,\;x-t_{\min}\,\bar v\big)$
  \Comment{final full jump to the endpoint}
\end{algorithmic}
\end{algorithm}

\begin{algorithm}[h]
\caption{GraspMeanFlow-JVP training (Stage 2)}
\label{alg:training_ot}
\begin{algorithmic}[1]
\Require dataset $\mathcal{D}$, encoder $f_{\mathrm{enc}}$, two-time
  field $f_\theta$, weight $\lambda_{\mathrm{mf}}$, dropout
  $p_{\mathrm{uncond}}$, Huber radius $c$
\While{not converged}
  \State Sample objects $\{\mathcal{P}^o\}_{o=1}^{B}$ with grasp sets
    $\{(R_0^{o,i}, x_0^{o,i})\}_{i=1}^{K} \sim \mathcal{D}$;\;\;
    $z^o \gets f_{\mathrm{enc}}(\mathcal{P}^o)$;\;\;
    $z^o \gets 0$ w.p.\ $p_{\mathrm{uncond}}$
  \State Sample $R_1^{o,i} \sim \mathrm{Unif}(\mathrm{SO}(3))$,\;
    $x_1^{o,i} \sim \mathcal{N}(0, I)$
  \For{each object $o$}
    \Comment{per-object OT coupling}
    \State $C^o_{ij} \gets
      w_R\lVert\log(R_1^{o,j}(R_0^{o,i})^{\top})^{\vee}\rVert^2
      + w_x\lVert x_1^{o,j} - x_0^{o,i}\rVert^2$
    \State $\sigma^o \gets
      \arg\min_{\sigma \in \mathfrak{S}_K} \sum_i C^o_{i\,\sigma(i)}$;\;\;
      $(R_1^{o,i}, x_1^{o,i}) \gets (R_1^{o,\sigma^o(i)}, x_1^{o,\sigma^o(i)})$
  \EndFor
  \State Sample $t \sim \mathrm{Unif}(0,1]$,\;\; $s \sim \mathrm{Unif}(0,t)$
  \State $\omega_t \gets \log(R_1 R_0^\top)^{\vee}$;\;\;
    $v_t \gets x_1 - x_0$
  \State $R_t \gets \exp\!\big(t\,[\omega_t]\big)\,R_0$;
    \;\; $x_t \gets (1{-}t)\,x_0 + t\,x_1$
  \State $(\omega^{\mathrm{avg}}_\theta, v^{\mathrm{avg}}_\theta)
    \gets f_\theta(t,t,R_t,x_t)$
    \Comment{flow-matching boundary}
  \State $\mathcal{L}_{\mathrm{fm}} \gets
    \lVert \omega^{\mathrm{avg}}_\theta - \omega_t \rVert^2
    + \lVert v^{\mathrm{avg}}_\theta - v_t \rVert^2$
  \State $\dot s \gets 0$;\;\; $\dot t \gets 1$;\;\;
    $\dot R_t \gets [\omega_t] R_t$;\;\; $\dot x_t \gets v_t$
    \Comment{trajectory tangent}
  \State $\big(\bar\omega_\theta, \bar v_\theta\big),\;
    \big(\dot{\bar\omega}_\theta, \dot{\bar v}_\theta\big)
    \gets \mathrm{jvp}\big(f_\theta;\;(s,t,R_t,x_t),\;
    (\dot R_t,\dot x_t,\dot s,\dot t)\big)$
  \State $\omega^{s\to t}_\theta \gets (t-s)\,\bar\omega_\theta$
  \State $\omega_{\mathrm{tgt}} \gets
    J\big(\mathrm{sg}(\omega^{s\to t}_\theta)\big)^{-1}\omega_t
    - (t-s)\,\dot{\bar\omega}_\theta$;\;\;
    $v_{\mathrm{tgt}} \gets v_t - (t-s)\,\dot{\bar v}_\theta$
  \State $\mathcal{L}_{\mathrm{mf}} \gets \|\bar{\omega}_{\theta}-\text{sg}(\omega_{\text{tgt}})\|^2
    + \lVert \bar v_\theta - \mathrm{sg}(v_{\mathrm{tgt}})\rVert^2$
    \Comment{Huber on the rotation residual}
  \State Adam step on
    $\mathcal{L}_{\mathrm{fm}} + \lambda_{\mathrm{mf}}
    \mathcal{L}_{\mathrm{mf}}$ with gradient-norm clipping;\; EMA update
\EndWhile
\end{algorithmic}
\end{algorithm}

\section{Evaluation Protocol}
\label{sec:protocol}

\paragraph{EMD.}
For each test instance, we generate as many grasps as the instance's ground-truth set contains and compute a Hungarian matching under the cost $$\sqrt{\lVert\log(R_g^\top R_{gt})^{\vee}\rVert_2^2+\lVert x_g-x_{gt}\rVert_2^2}$$ with translations in meters (the benchmark loader's $\times 8$-normalized coordinates rescaled back). Scores are averaged over the three rotations and all objects of a category, then macro-averaged over categories. Our implementation agrees with EquiGraspFlow's released \texttt{SE3\_geodesic\_dist} to four decimals on identical samples, and running their model at its native setting reproduces their published EMD within $0.004$.

\paragraph{NFE-1 success.}
The main success table starts at NFE~2; for completeness, average NFE-1 success rates are: SE(3)-DiF $6.7\%$, BRIDGER $46.0\%$, EquiGraspFlow $17.5\%$, GraspMeanFlow (Euler) $18.7\%$. BRIDGER's NFE-1 spike (and its NFE-2 collapse and rising EMD beyond NFE~5) traces to its released sampler, which suppresses noise over the final $30\%$ of the trajectory and amplifies the learned drift fivefold; at $T{=}1$ this yields a single strongly mode-seeking step from its heuristic prior. Consistent with this, its sample spread doubles from $T{=}1$ to $T{=}2$ and then shrinks monotonically. At its native $T{=}40$ setting, BRIDGER averages $92.6\%$.

\section{Simulation Details}
\label{sec:sim}
Success is evaluated with a floating-gripper lift-and-hold test ported from SE(3)-DiffusionFields' public \texttt{isaac\_evaluation}: the Panda hand closes with constant torque, gravity is applied to the object, and the grasp succeeds if the object stays near its held pose. Table~\ref{tab:simparams} lists all simulator settings, which are identical for every method.

\begin{table}[h]
\centering
\small
\setlength{\tabcolsep}{6pt}
\begin{tabular}{ll}
\toprule
Simulation step / substeps & $1/250$\,s~/~2 \\
Solver position iterations & 25 \\
Finger closing effort / steps & 25\,N~/~120 \\
Gravity / hold duration & $9.81\,\text{m/s}^2$~/~150 steps ($0.6$\,s) \\
Success threshold & $0.30$\,m displacement \\
Friction coefficient & 3 \\
Convex decomposition & V-HACD, $2{\times}10^5$ resolution \\
Grasps per instance & 100 \\
Rotated instances (B/L/M/P) & 30~/~105~/~66~/~48 \\
\bottomrule
\end{tabular}
\caption{Isaac Gym lift-test settings, identical for all methods. Each
method--NFE cell simulates $24{,}900$ grasps.}
\label{tab:simparams}
\end{table}

The harness is validated in both directions: ground-truth ACRONYM grasps score $97.8\%$ on average while negative controls score ${\sim}0\%$. It is insensitive to judge settings --- halving the success threshold to $0.15$\,m or reducing friction to 2 moves every method by at most one point and changes no ranking. One mesh whose convex decomposition fails in PhysX is excluded identically for all methods. Removing the test-time $\mathrm{SO}(3)$ rotations changes aggregate metrics negligibly for both equivariant models.

\section{Full Quantitative Results}
\label{sec:dense}
Table~\ref{tab:emd} reports the category-averaged EMD over the NFE grid for both configurations. GMF-SG is lowest through NFE~5; its single evaluation already matches what EquiGraspFlow reaches with two to five ($0.418$ against $0.420$ and $0.416$), and EquiGraspFlow needs ten evaluations to draw level ($0.002$) and twenty to lead by $0.012$. GMF-JVP trails GMF-SG by $0.04$ to $0.06$ at every budget, and both configurations stay far below every baseline at NFE~1 and~2.

\begin{table}[h]
\centering
\small
\setlength{\tabcolsep}{4.0pt}
\begin{tabular}{@{}lccccc@{}}
\toprule
Method & NFE 1 & 2 & 5 & 10 & 20 \\
\midrule
SE(3)-DiF       & 0.694 & 0.628 & 0.512 & 0.456 & 0.449 \\
BRIDGER         & 0.717 & 0.657 & 0.547 & 0.598 & 0.650 \\
EquiGraspFlow   & 0.734 & \textbf{0.420} & \textbf{0.416} &
                  \textbf{0.380}$^{*}$ & \textbf{0.369}$^{*}$ \\
\midrule
GMF-SG (Euler)  & \textbf{0.418}$^{*}$ & \textbf{0.399}$^{*}$ &
                  \textbf{0.390}$^{*}$ & \textbf{0.382} & \textbf{0.381} \\
GMF-JVP (Euler) & \textbf{0.466} & 0.456 & 0.432 & 0.425 & 0.422 \\
\bottomrule
\end{tabular}
\caption{EMD on the randomly rotated four-category test sets, averaged over
categories, with Euler integration for both GraspMeanFlow variants. Lower is
better. Best and second-best in each column are in bold; the best is marked
with $^{*}$.}
\label{tab:emd}
\end{table}

Table~\ref{tab:success} breaks success down by category and NFE; Figure~\ref{fig:succ_cat5} plots the NFE-5 column. The exp-$\mathrm{SO}(3)$
sampler holds the best Bowl results at NFE~5, 10, and~20.

Success under Euler integration, for both configurations and at every budget, is
listed in Table~\ref{tab:posttrain}, alongside the post-training results.

\section{Additional Ablations}
\label{sec:ablation}

\paragraph{Objective ablation.}
Table~\ref{tab:objablation} compares one-step EMD for five objectives under a matched lighter protocol (per-category models warm-started from a shared $\alpha$-annealed initialization, three seeds). Adding the semigroup term improves the JVP MeanFlow objective on every category, and the flow-matching boundary is the most stable anchor. The MF column uses the body-frame trajectory tangent paired with the left Jacobian, which diverges on the rotationally symmetric Bowl; the frame-consistent inverse-Jacobian target used by GMF-JVP is described in the main text.

\paragraph{Coupling.}
The two reported configurations differ in both the consistency term and the
prior--data coupling: GMF-SG uses independent pairing and GMF-JVP per-object
optimal transport (Appendix~\ref{sec:ot-coupling}). The four-way comparison
that would separate the two effects was outside our compute budget, so the gap
between the two configurations should not be attributed to either factor
alone. We report them as two instantiations of one framework rather than as a
controlled comparison of consistency terms. The comparison against
EquiGraspFlow is unaffected for GMF-SG, which shares the baseline's independent
coupling.

\begin{table}[h]
\centering
\small
\setlength{\tabcolsep}{4pt}
\begin{tabular}{lccccc}
\toprule
& MF & MF+sg & $\alpha$ & $\alpha$+sg & FM+sg \\
\midrule
Mug & 0.634 & 0.599 & 0.568 & 0.551 & \textbf{0.549} \\
Laptop & 0.640 & 0.610 & 0.592 & 0.557 & \textbf{0.509} \\
Pencil & 0.553 & 0.503 & 0.420 & 0.411 & \textbf{0.409} \\
Bowl & 1.110 & 0.978 & \textbf{0.513} & 0.763 & 0.534 \\
\midrule
Average & 0.734 & 0.673 & 0.523 & 0.571 & \textbf{0.500} \\
\bottomrule
\end{tabular}
\caption{One-step EMD by training objective (MF = JVP MeanFlow identity, $\alpha$ = $\alpha$-annealed instantaneous--average interpolation, sg = semigroup consistency; three-seed means under the ablation protocol).}
\label{tab:objablation}
\end{table}

\paragraph{Boundary variants.}
Replacing the flow-matching boundary with an endpoint-prediction boundary is competitive overall and notably better on the rotationally symmetric Bowl (one-step EMD ${\sim}0.40$ vs.\ ${\sim}0.53$), suggesting endpoint regression suits symmetric, multimodal targets; we keep the flow-matching boundary as the default for its robustness across categories.

\paragraph{Warm-up and training length.}
Training the decomposed objective from scratch is on par with the $\alpha$-annealed warm-up (NFE-20 success $93.4\%$ vs.\ $94.2\%$ under the unrotated Euler ablation protocol; EMD curves overlap), so the warm-up is optional. Extending training from 120k to 200k steps leaves EMD flat and moves NFE-20 success by at most $1.6$ points, so all reported results use the 120k checkpoint.

\section{Latency Details}
\label{sec:latency}

Table~\ref{tab:latency} reports wall-clock latency per NFE setting,
measured on a single A100 with batch size one, timing model inference
only (dataset loading, mesh handling, and simulation excluded) after
warm-up. Because GMF-SG and GMF-JVP share the same network, their
latency is identical at every step count and is reported once.

\begin{table}[h]
\centering
\small
\setlength{\tabcolsep}{4.5pt}
\begin{tabular}{lccccc c}
\toprule
& NFE 1 & 2 & 5 & 10 & 20 & native \\
\midrule
GraspMeanFlow & 31.5 & 32.5 & 41.4 & 56.3 & 87.0 & -- \\
EquiGraspFlow & 34.1 & 33.8 & 44.7 & 62.8 & 99.0 & 315.0 \\
BRIDGER & 8.3 & 10.7 & 17.1 & 28.6 & 51.4 & 97.1 \\
SE(3)-DiF & 10.9 & 16.6 & 33.7 & 63.1 & 119.2 & 1338.7 \\
\bottomrule
\end{tabular}
\caption{Wall-clock latency (ms, A100, batch~1, model inference only,
after warm-up). NFE columns run every method's own sampler at matched
step counts; \emph{native} is the released default configuration, at
which each baseline's published quality is obtained: RK-MK with $20$
steps ($80$ field evaluations) for EquiGraspFlow, $40$ steps for
BRIDGER, and annealed Langevin with $240$ evaluations for
SE(3)-DiffusionFields. The two GraspMeanFlow configurations share a
network and therefore a latency. The point-cloud encoder
(${\sim}29$\,ms) is evaluated once regardless of step count and
dominates the equivariant models at small NFE; the baselines use smaller
networks and are cheaper per step. In batched mode (batch~100),
GraspMeanFlow takes $0.34$\,ms/grasp at NFE~1 (a 100-candidate set
in ${\sim}34$\,ms).}
\label{tab:latency}
\end{table}

\section{Post-Training with Success-Filtered ReFlow}
\label{sec:posttrain}
\begin{table}[h]
\centering
\small
\setlength{\tabcolsep}{5pt}
\begin{tabular}{llcccc}
\toprule
Model & Sampler & NFE 2 & 5 & 10 & 20 \\
\midrule
GMF-SG            & Euler                & 36.8 & 68.3 & 81.7 & 94.7 \\
GMF-SG            & exp-$\mathrm{SO}(3)$ & 19.2 & 85.5 & 90.6 & 91.1 \\
GMF-SG + ReFlow   & Euler                & 48.6 & 83.0 & 89.7 & 93.4 \\
GMF-SG + ReFlow   & exp-$\mathrm{SO}(3)$ & 7.8 & 76.8 & 90.2 & 91.8 \\
\midrule
GMF-JVP           & Euler                & \textbf{53.9} & 81.7 & 93.2 & 94.6 \\
GMF-JVP           & exp-$\mathrm{SO}(3)$ & 15.4 & \textbf{91.1} & \textbf{96.3} & \textbf{95.6} \\
GMF-JVP + ReFlow  & Euler                & 47.6 & 83.0 & 92.0 & 92.2 \\
GMF-JVP + ReFlow  & exp-$\mathrm{SO}(3)$ & 8.2 & 76.7 & 92.4 & 94.3 \\
\bottomrule
\end{tabular}
\caption{Four-category average success rate (\%) for every sampler and
post-training combination. Post-training improves the Euler sampler of
GMF-SG but degrades the endpoint-style sampler of both configurations,
and on GMF-JVP it helps at no budget except Euler at NFE~5.}
\label{tab:posttrain}
\end{table}

The model reported in the main text is trained only on the ACRONYM grasp
annotations. A natural question is whether the simulator itself can be used as a
training signal: the lift test provides a binary label for any generated grasp,
so the model can be refined on its own successful samples. We report such an
experiment here. It improves the Euler sampler substantially at small step
budgets but does not surpass the endpoint-style sampler of
Section~\ref{sec:sampling}, which is why the main text reports the pre-trained
model only.

\paragraph{Procedure.}
Starting from the converged Stage-2 checkpoint, we generate (noise, grasp) pairs
on the training objects with 20-step Euler sampling, discard the pairs whose
grasp fails the simulated lift test of Appendix~\ref{sec:sim} ($94.3\%$ pass),
and continue training for 15k steps on the surviving pairs. This is a
success-filtered instance of ReFlow~\citep{liu2022reflow}: the coupling between
prior and data samples is replaced by the model's own transport, so the
straightened paths are those the model already traverses successfully. All other
settings follow the Stage-2 column of Table~\ref{tab:hyper}, except that the
batch uses 8 objects instead of 4 and the coupling is self-generated rather than
independent. The filtered pairs are regenerated once at the start of the stage
and not refreshed.

\paragraph{Results.}
Table~\ref{tab:posttrain} reports every combination of the two configurations
with the two samplers, before and after post-training.

On GMF-SG the Euler sampler gains $11.8$, $14.7$, and $8.0$ points at NFE~2, 5,
and~10. This is the effect ReFlow is designed for: independently drawn (noise,
grasp) pairs induce curved transport that few steps cannot follow, and
re-pairing each grasp with the noise the model itself maps to it straightens
those paths. At NFE~20 the gain reverses, costing $1.3$ points, consistent with
the same reading --- once the solver has enough steps, the restricted coupling
is a loss of coverage rather than a gain in straightness.

On GMF-JVP the stage gives nothing. Its per-object optimal-transport assignment
(Appendix~\ref{sec:ot-coupling}) already supplies a straightened coupling during
training, so the benefit ReFlow offers has already been collected; only the cost
of an extra stage on self-generated data remains, and every budget except Euler
at NFE~5 gets worse.

That cost lands hardest on the endpoint-style sampler, in both configurations.
The exp-$\mathrm{SO}(3)$ sampler commits the rotation to the currently predicted data
endpoint at every step, so it needs the interval field's endpoint prediction to
be accurate across the whole interval; refitting on a finite set of the model's
own generations, built once and never refreshed, is precisely a loss of endpoint
accuracy. Euler is more forgiving because it consumes only the local average
velocity over a short step. Accordingly, at NFE~5 the endpoint sampler drops
from $85.5$ to $76.8$ on GMF-SG and from $91.1$ to $76.7$ on GMF-JVP.

We therefore report the pre-trained models throughout the main text, and record
this experiment as a negative result: success-filtered ReFlow helps only a model
whose coupling is not already transport-optimal, and it pays for the attempt by
blurring the endpoint predictions the sampler we report depends on.


\paragraph{Caveat.}
The filtering threshold is the simulator's binary outcome, which discards all information about \emph{how} a grasp failed. A soft-weighted variant, or one that refreshes the generated pairs periodically rather than once, may behave differently; we did not explore either.

\end{document}